%% file: main.tex
\documentclass{article}

    \PassOptionsToPackage{sort, numbers, compress}{natbib}

    \usepackage[preprint]{neurips_2021}

\usepackage[utf8]{inputenc} 
\usepackage[T1]{fontenc}    
\usepackage{hyperref}       
\usepackage{url}            
\usepackage{booktabs}       
\usepackage{amsfonts}       
\usepackage{nicefrac}       
\usepackage{microtype}      
\usepackage{xcolor}         
\usepackage{enumitem}

\usepackage{microtype}
\usepackage{graphicx}
\usepackage{subfigure}
\usepackage{booktabs} 

\usepackage{amssymb}
\usepackage{bbm}
\usepackage{amsmath}
\usepackage[noend]{algorithmic}
\usepackage{algorithm}
\usepackage{amsthm}
\usepackage{color}

\usepackage{hyperref}

\usepackage{defns}
\usepackage{kk}

\title{PAC Best Arm Identification Under a Deadline}

\author{%
  Brijen Thananjeyan \hspace{0.2in} Kirthevasan Kandasamy \hspace{0.2in} Ion Stoica \hspace{0.2in} Michael I. Jordan\\
  \textbf{Ken Goldberg} \hspace{0.2in} \textbf{Joseph E. Gonzalez}\\
  University of California, Berkeley
  
}

\begin{document}

\maketitle

\input{abstract}

\input{introduction}
\input{problem_statement}
\input{method}
\input{lowerbound}
\input{experiments}

\input{conclusion}

\bibliography{refs,kky}
\bibliographystyle{plainnat}


\newpage

\clearpage

\appendix
\input{appendix}

\input{app_lower_bound}

\end{document}

%% file: abstract.tex
\begin{abstract}
We study $(\epsilon, \delta)$-PAC best arm identification, where a decision-maker must identify an $\epsilon$-optimal arm with probability at least $1 - \delta$, while minimizing the number of arm pulls (samples). Most of the work on this topic is in the sequential setting, where there is no constraint on the \emph{time} taken to identify such an arm; this allows the decision-maker to pull one arm at a time. In this work, the decision-maker is given a deadline of $T$ rounds, where, on each round, it can adaptively choose which arms to pull and how many times to pull them; this distinguishes the number of decisions made (i.e., time or number of rounds) from the number of samples acquired (cost). Such situations occur in clinical trials, where one may need to identify a promising treatment under a deadline while minimizing the number of test subjects, or in simulation-based studies run on the cloud, where we can elastically scale up or down the number of virtual machines to conduct as many experiments as we wish, but need to pay for the resource-time used. As the decision-maker can only make $T$ decisions, she may need to pull some arms excessively relative to a sequential algorithm in order to perform well on all possible problems. We formalize this added difficulty with two hardness results that indicate that unlike sequential settings, the ability to adapt to the problem difficulty is constrained by the finite deadline. We propose \algname ~(\algabbr{}), a novel algorithm for this setting and bound its sample complexity, showing that \algabbr{} is optimal with respect to both hardness results. We present simulations evaluating \algabbr{} in this setting, where it outperforms baselines by several orders of magnitude.
\end{abstract}

%% file: introduction.tex
\section{Introduction}
\label{sec:intro}


\newcommand{\insertFigIllus}{
\begin{figure}[t!]
    \centering
    \includegraphics[width=4.01in]{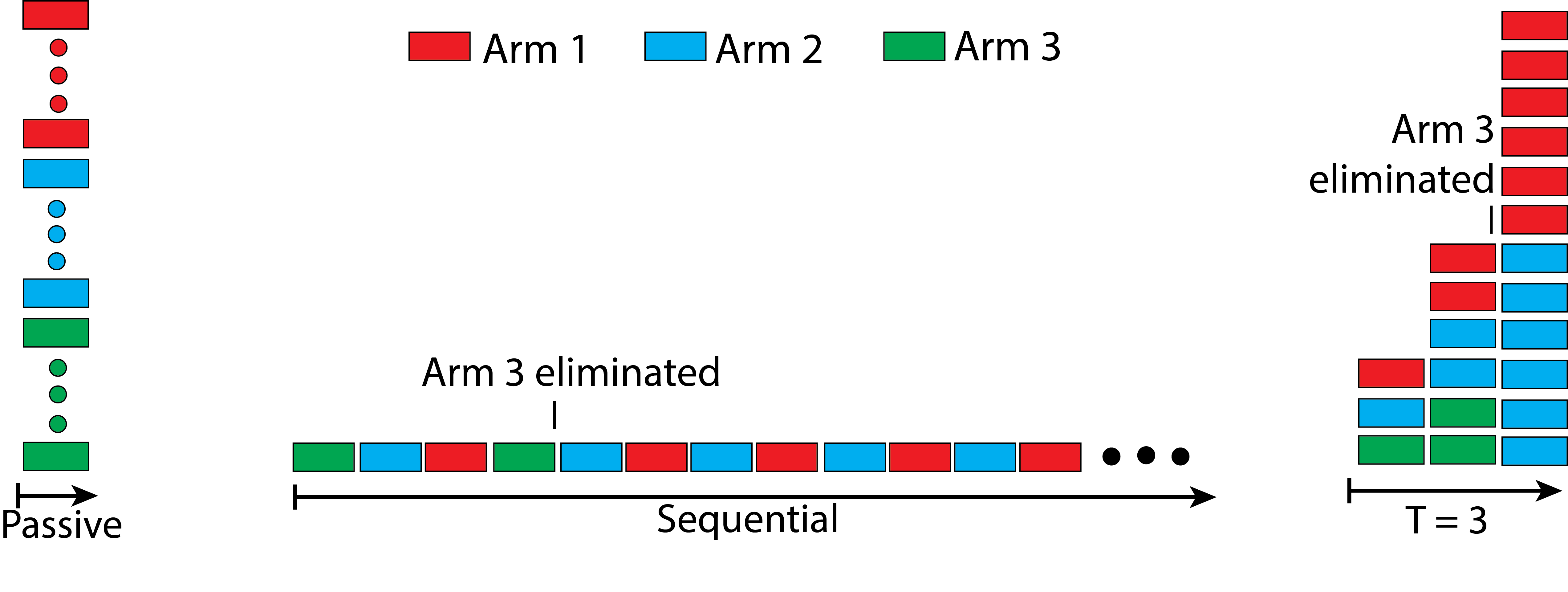}
\vspace{-0.15in}
    \caption{\small
An $(\epsilon, \delta)$-PAC BAI problem on a bandit model with three arms where
$\mu_1>\mu_2 \gg \mu_3$, i.e., the third arm is `easy' as it can be differentiated with a few
samples.
\emph{Left:}. When $T = 1$, all arms must be pulled many ($\bigOtilde(n\epsilon^{-2})$) times
so as to ensure we do well even on the hardest problems.
This can result in over-pulling the third arm.
\emph{Middle:} 
When $T=\infty$, the third arm can be eliminated after a minimal amount of pulls;
however, this can take a long time (number of rounds).
\emph{Right:}
When $T<\infty$, the number of pulls for arm $3$ is smaller than $T=1$ but larger than $T=\infty$.
However, it finishes under the deadline.
\label{fig:settingillus}}
\vspace{-0.15in}
\end{figure}
}

In best arm identification (BAI), a decision-maker draws samples from a \emph{bandit model}
$\nu = (\nu_1,\dots,\nu_n)$ of $n$ arms, where,
upon pulling arm $i$, it receives a stochastic reward drawn from a real-valued
$\sigma$ sub-Gaussian distribution $\nu_i$, with mean $\mui$.
The goal is to identify the best or at least a
good arm, i.e., an arm with large $\mui$.
In the $(\epsilon,\delta)$-PAC (probably approximately correct) version of this problem,
we wish to identify an $\epsilon$--optimal
arm with probability at least $1-\delta$,
 while minimizing the  number of arm pulls (i.e. number of samples or cost).
Most of the literature on this topic is restricted to
the sequential setting, where the decision-maker adaptively draws samples
one at a time.
In this work, we instead study PAC BAI when the decision-maker needs to complete
its experimentation under a deadline of $T$ rounds.
In order to be able to fulfill the $(\epsilon,\delta)$ performance requirement,
executing multiple arm pulls per round is allowed, where this decision can be made adaptively
based on past information.
This setting has recently received attention in the
literature~\citep{agarwal2017learning,jin2019efficient}, driven by emerging opportunities
and challenges in various applications. We list some below:

\textbf{Clinical trials:}
Consider identifying, among $n$ candidates, a good vaccine for a viral disease.
Each arm pull consists of administering the vaccine to a test subject
and monitoring the subject's health for a period of time, say two months.
Given  $(\epsilon, \delta)$ performance requirements,
we wish to identify the candidate in under a year (i.e. $T=6$),
while reducing the number of arm pulls for ethical reasons and
to reduce cost.
Similar use cases arise when conducting experiments to identify a good candidate
in drug discovery or in materials science, where high-throughput experimental platforms
can be used to simultaneously
conduct a large number of experiments in parallel, but the number of experiments (samples)
needs to be minimized to reduce the cost of reagents~\citep{dave2020autonomous,broach1996high}.

\insertFigIllus

\textbf{Configuration tuning on the cloud:}
BAI is used for configuration tuning tasks in
statistical model selection, simulation-based scientific
studies, and optimizing real-time
systems~\citep{li2017hyperband,tegmark06lrgs,venkataraman2016ernest,liaw2019hypersched,%
kandasamy2020tuning}. 
Here, each arm pull consists of running a computer-based experiment using a single
resource (e.g., a CPU or GPU).
While past work assumes a single resource or a fixed number of parallel resources,
with the advent of cloud computing we have the opportunity to elastically
scale up or down the resources we use for arm pulls,
while paying for the total resources used~\citep{misra2021rubberband}.
We wish to minimize the number of arm pulls in order to minimize our payments to the cloud provider.


This setting can be viewed as an intermediary between \emph{passive}
(completely non-adaptive, i.e., $T=1$)
and \emph{sequential} (completely adaptive, i.e., $T=\infty$) BAI.
See Fig.~\ref{fig:settingillus} for an illustration.
In the passive case, we have no option but to pull each arm
$\epsilon^{-2}$ times, since, in the hardest case, all sub-optimal arms could be
arbitrarily close to being $\epsilon$ away from the optimal arm\footnote{
Recall that
$\bigThetatilde(\Delta^{-2})$ samples are both sufficient and necessary to
distinguish between two $1$-sub-Gaussian random variables whose means are $\Delta$
apart with constant
probability~\citep{chernoff1972sequential,kaufmann2016complexity,farrell1964asymptotic}.
}.
This requires $\bigOtilde(n\epsilon^{-2})$ samples regardless of the difficulty of the specific
problem (bandit model) $\nu$.
In the sequential setting, an agent may pull one arm at a time and incorporate
information from previous pulls to decide which arm to pull next.
In this case,  the sample complexity on a bandit model $\nu$ is known to be
$\bigOtilde(\Hcompl(\nu))$, where 
$\Hcompl(\nu)$, defined below, characterizes the difficulty of problem
 $\nu$~\citep{kaufmann2016complexity,mannor2004sample}.
We have:
\vspace{-0.025in}
\begin{align*}
\hspace{-0.07in}
\Hcompl(\nu) = \sum_{i=1}^n\Niopt,
\hspace{0.30in}
 \Niopt =
  \begin{cases}
   \epsilon^{-2}        & \text{if } \Delta_i < \epsilon \\
   (\Delta_i + \epsilon)^{-2}    & \text{otherwise}
  \end{cases},
\hspace{0.30in}
 \Delta_i =
  \begin{cases}
   \mu_{[1]} - \mu_i        & \text{if } i \neq [1] \\
   \mu_{[1]} - \mu_{[2]}    & \text{if } i = [1]
  \end{cases}.
\vspace{-0.025in}
\label{eqn:gapsdefn} \numberthis
\end{align*}
Here\footnote{%
Prior results for sequential $(\epsilon,\delta)$-PAC
BAI, use slightly different expressions for $\Hcompl(\nu)$,
e.g. Remark 5 in~\citet{kaufmann2016complexity}.
These expressions are not fundamentally different from~\eqref{eqn:gapsdefn} with the upper
and lower
bounds only differing by constants.
}, the arm with the $j$-th highest mean is denoted by $[j]$;
$\Delta_i$ denotes the gap between the $i$\ssth arm and the best arm, except when $i$ is the
best arm in which case it denotes the gap between itself and the next best arm;
$\Niopt$ quantifies the (order of the) number of samples required from arm $i$ to determine if
it is optimal.
This shows that when compared to the passive setting,
 sequential algorithms can be significantly cheaper in terms of
the number of samples (arm pulls), as they are
able to adapt to problem difficulty.
In particular, when the gap
$\Delta_i$ is large, $\Niopt \ll \epsilon^{-2}$.
Unfortunately, this takes as many rounds as the number of samples which can be prohibitive
for practical use cases.

The $1<T<\infty$ case, which reflects most practical settings for BAI,
was surprisingly only recently studied in the
literature~\citep{jin2019efficient,agarwal2017learning}.
Here, we are allowed to be adaptive, but to a limited extent, and
therefore need to be prudent in how we use this adaptivity.
If we choose to invest few arm pulls per arm in the early rounds in the hope that we can
eliminate some easy (large $\Delta_i$) arms at low cost, we risk not eliminating
arms with intermediate difficulty.
In contrast, if we allocate too many pulls for all arms early on, we may have already spent too much
to eliminate the easy arms.
We show the number of samples required in this setting falls in
between the results for the $T=1$ and $T=\infty$ cases.

\textbf{Our contributions}:
\textbf{\emph{(i)}}
In Section~\ref{sec:setup},
we formalize the problem, and in Section~\ref{sec:method}, we describe \algname{} (\algabbr), a
racing-style algorithm for this problem.
We provide two upper bounds on its sample complexity.
The first  partitions the problem space into progressively harder problems,
and bounds the number of samples required by \algabbr{} uniformly for each partition.
This partitioning, which depends on  $\epsilon$ and $T$,
approaches individual problems as $T\rightarrow \infty$.
Our second upper bound is problem-dependent, showing that no more than
$\bigOtilde(\epsilon^{\nicefrac{-2}{T}}\Hcompl(\nu))$ samples are required on a bandit model $\nu$.
\textbf{\emph{(ii)}}
In Section~\ref{sec:lowerbound}, we provide two types of hardness results for this
problem which match the two upper bounds presented above, thus establishing optimality
of \algabbr.
The first is a bound on the sample complexity of hardest problem in each set of the
partitioning.
The second establishes that the worst-case ratio between the sample complexity of a finite $T$
algorithm and $\Hcompl(\nu)$ can be as large as 
$\epsilon^{\nicefrac{-2}{T}}$.
\textbf{\emph{(iii)}}
In Section~\ref{sec:experiments},
we corroborate these theoretical insights on simulation experiments and show that the proposed
algorithm outperforms other baselines by several orders of magnitude.

\subsection{Related work}\label{sec:related_work}
\vspace{-0.10in}

Multi-armed bandits are a popular framework to model the exploration-exploitation tradeoffs
that arise in decision-making under
uncertainty~\cite{thompson33sampling,robbins52seqDesign,auer03ucb}.
In such problems,
a decision-maker must adaptively sample arms from a bandit model $\nu$, so as to fulfill a certain goal.
There is a long line of work on best arm identification, where, at the
end of the sampling process, the 
decision-maker must  output a prediction for the arm with the highest (or a high)
mean~\cite{karnin2013almost,bubeck2009pure,jamieson2014lil,gabillon2012best,%
garivier2016optimal,chen2015optimal,bubeck2013multiple,russo2016simple}.
In addition to the PAC version,
there are other variants for BAI:
the $\delta$-probably correct (a.k.a fixed confidence) version identifies the best arm with
probability at least $1-\delta$,
the fixed budget version minimises the probability of mis-identifying the best arm under a
budget of arm pulls,
and some formulations minimize a loss function based on the probability of
selecting an arm and its gap $\Delta_i$~\eqref{eqn:gapsdefn}.
In addition, prior work has also studied top-$k$ variants in all of the above formulations where we
wish to identify the best $k$ arms instead of simply the best arm~\cite{jun2016top,jin2019efficient,grover2018best,kaufmann2016complexity}.
All of the ideas in this paper carry through to the top-$k$ setting, but we focus on 
top-$1$ for simplicity.


Our setting is distinctly different from other BAI work where arms can be pulled in parallel.%
~\citet{jun2016top} and~\citet{grover2018best} study batch best arm
identification where the agent can pull up to a fixed batch size $b$ of pulls in parallel, and the
goal is to minimize the number of rounds taken,
to identify the best arm with probability at least $1-\delta$.%
~\citet{thananjeyan2021overcoming} consider a slighlty different
version where there is a fixed amount of a resource
to pull the arms, but the
time taken to execute pulls is a function of the number of resources assigned to it.
In contrast to
this prior work, in our setting, both  time (number of rounds) and
failure probability $\delta$ are fixed, but we may execute a variable number of arm pulls
on each round
with the goal of satisfying the $(\epsilon, \delta)$ requirement while minimizing the cost.

To the best of our knowledge, settings similar to ours have only been studied before in a small number of papers, including work by
~\citet{agarwal2017learning} and ~\citet{jin2019efficient}.
We will discuss their results in further detail in Sections~\ref{sec:method} and~\ref{sec:lowerbound}.
The main difference in our results relative to these works is that we can adapt
to problem difficulty more effectively.
Indeed, both of the above papers provide  $\bigThetatilde(n\epsilon^{-2})$ upper and lower
bounds on the worst-case sample complexity, which is obtained for
the hardest problem in the problem class.
Additionally, the algorithm in~\citet{agarwal2017learning}
requires knowledge of the smallest arm gap $\Delta_{[1]}$, which can be a signficant
limitation in practice.
Finally, as we will see in Section~\ref{sec:experiments},
our algorithm empirically performs much better.

%% file: problem_statement.tex
\section{Problem setup}
\label{sec:setup}

First, let us describe the environment that is studied in this paper. We have $n$ arms, and refer to
arms by their index $i\in [n]$. Arm $i$ is associated with a distribution $\nu_i$
with mean $\mu_i = \mathbb{E}_{X\sim\nu_i}\left[X\right]\in [0,1]$, and pulling the arm generates an independent sample
from this distribution.
Unless otherwise stated, we will let $\mathcal{P}$ be the set of $n$ $\sigma$-sub-Gaussian distributions with mean in $[0,1]$, and $\nu\in \mathcal{P}$ be the distribution corresponding to a set of $n$ arms.
Recall that the arm with the $i$-th highest mean is denoted by $[i]$, so that
$\mu_{[1]} \leq \mu_{[2]}  \leq \dots \leq \mu_{[n]}$.
Let the gaps $\{\Delta_i\}_{i\in[n]}$ be as defined in~\eqref{eqn:gapsdefn}.

In this paper, we are given a maximum error probability $\delta\in(0,1)$, error tolerance
$\epsilon\in(0,1)$, and a deadline $T\in\NN_+$ on the number of rounds.
Our goal is to find an $\epsilon$-optimal arm\footnote{%
While some BAI work in the sequential setting require that the algorithm output the best arm,
we study an $\epsilon$-optimal version to enable the problem to be
feasible. For example, when $T = 1$, any algorithm will simply pull all the arms some
number of times $N$ without prior
information; we can always render exact best arm identification infeasible
for any algorithm by choosing a hard problem where $\Delta_{[1]} \in \littleO(1/\sqrt{N})$.
}
with probability at
least $1 - \delta$ in at most $T$ rounds while minimizing the number of arm pulls (number of
samples or cost).
An $\epsilon$-optimal arm $i$ is one that has mean that is within $\epsilon$ of $\mu_{[1]}$:
$\mu_i > \mu_{[1]} - \epsilon$.

An algorithm is defined by a \emph{sampling rule}  and a
\emph{recommendation rule}.
The sampling rule $\{A_t\}_{t=1}^T$, determines the arms to be sampled and the number of times they
need to be sampled.
Here, $A_t = \{N_{i,t}\}_{i=1}^n$, where $N_{i,t}$ is the
number of times to sample arm $i$ at round $t$.
At the end of round $t$, the algorithm receives
observations $O_t = \bigcup_{i=1}^n \{X_{i,t,j}\}_{j=1}^{N_{i,t}}$ where $X_{i,t,j}$ is the $j$-th
sample at time $t$ from arm $i$.
$A_t$ can depend on past observations and is $\Fcal_{t-1}$ measurable, where
$\Fcal_{t-1} =\sigma(\{(A_s,O_s)\}_{s=1}^{t-1})$ is the $\sigma$-field
generated by observations up to round $t-1$.
The recommendation rule $\recrule$ outputs an arm in $[n]$, and is $\Fcal_{T}$ measurable.
Denote $\widetilde{N}_{i, t} = \sum_{k=1}^t N_{i,
k}$ to be the number of times arm $i$ is pulled through time $t$
and $\Ntilde_{t}=\sum_{i=1}^n\Ntilde_{i,t}$ to be the total number of pulls in $t$ rounds.



\paragraph{Challenges:}
We begin by providing an intuitive explanation of the challenges in our setting.
Consider a simple two-armed bandit
model $\nu=(\nu_1,\nu_2)$ where $\nu_1=\Ncal(\mu_1, \sigma^2)$ has a known mean
of $\mu_1=1/2$ while $\nu_2=\Ncal(\mu_2, \sigma^2)$ with $\mu_2 < 1/2-\epsilon$
(although this is unknown) so that
$\Delta_{[1]} = 1/2 - \mu_2$.
We will take $T=2$ and assume $\epsilon \ll 1/2$.
As $\mu_1$ is known,
an algorithm for this setting pulls arm $2$ some number of times on the first round,
and then uses that information to determine how many more times to pull in the second round.
Assume an algorithm pulled $x$ number of times on the first round.
If the problem was very easy, i.e., 
$x
\gg
\bigOtilde((\Delta_{[1]} +\epsilon)^{-2})$, then we have already over-pulled
on the first round for this problem.
On the other hand, if $\mu_2$ was such that $x$ number of pulls was insufficient to determine
it was sub-optimal, then the algorithm will need to pull at least $\bigOmega(\epsilon^{-2}) - x$
times at round 2 to ensure that it satisfies the $(\epsilon,\delta)$ requirements
even on the hardest problems (i.e., very small $\Delta_{[1]}$).
If however, only slightly more pulls than $x$ were necessary on this problem,
we will have over-pulled again, this time in the second round.
We make this intuition rigorous in Section~\ref{sec:lb_proof_sketches} 
and the proof of Theorem~\ref{thm:lb_n2_T2}.
Ideally, we would like to pull exactly
$\bigThetatilde((\Delta_{[1]} + \epsilon)^{-2})$ times\footnote{%
We have $\bigO((\Delta_{[1]} + \epsilon)^{-2})$  and not $\bigO(\Delta_{[1]}^{-2})$
since, we only need to
verify $\mu_2<1/2+\epsilon$ in order to output $\recrule=1$ as an $\epsilon$-optimal arm.}.
While a sequential algorithm can achieve this by executing the pulls one at a time, this
is not possible when we only have finite rounds.

This example also illustrates why problem-dependent hardness results are not possible
in our setup.
Hence, any hardness result will necessarily need to consider the hardness over a class of
problems.
However, we find that using a worst case sample complexity of $\bigThetatilde(n\epsilon^{-2})$
is warranted only when $T=1$, i.e., the passive case.
When  $1<T<\infty$, we are able to adapt to problem difficulty, but as explained above,
this ability is necessarily constrained by the finite deadline.

\paragraph{Summary of Results:}
Our results in this regard come in two flavors:
\vspace{-0.07in}
\begin{enumerate}[leftmargin=0.25in]
\item
First, we show that we can partition the problem space in a way that there is a partial
ordering between different sets in the partition.
Our proposed algorithm will require fewer samples for problems that are easier in this partial
ordering (Theorem~\ref{thm:partitioning_upper_bound}).
We complement this with a matching hardness result (Theorem~\ref{thm:partitioning})
showing that the above sample complexity
matches that of the hardest problem in each partition.
This partitioning is given in Definition~\ref{defn:partitioning}.
\vspace{-0.03in}

\item
Second, we consider the ratio $\Ntilde_T/\Hcompl(\nu)$,
which is the total number of samples required by a $T$ round algorithm, relative to the
problem complexity $\Hcompl(\nu)$.
Intuitively, if this ratio is uniformly small over all problems, then a $T$-round algorithm
does not do significantly worse than a sequential one on any problem.
EBR achieves a ratio of at most $\bigOtilde(\epsilon^{-2/T})$ on any problem
(Theorem~\ref{thm:fixed_b0}).
While we are unable to provide a completely matching hardness result, we provide two partial
results which suggest that this ratio cannot be improved in general.
The first result (Theorem~\ref{thm:lb_n2_T2}) shows that for the special case of $T=2, n=2$,
the worst case ratio over all problems in
$\Pcal$ could be as large as $\bigOmega(\epsilon^{-2/T})$.
The second result (Theorem~\ref{thm:lb_restricted_class}) establishes the same lower bound for
$n=2$ and general
$T$, but for a restricted class of algorithms.

\end{enumerate}

\vspace{0.05in}
\begin{definition}[\textbf{A partitioning of the problem space}]
{
\label{defn:partitioning}
Let $\gamma \in [T]^n$ be an index for each set in the partition, so that
$\mathcal{P} = \bigcup_{\gamma \in [T]^n} \mathcal{P}_\gamma$ and
$\Pcal_\gamma \cup \Pcal_{\gamma'}=\emptyset$ for $\gamma\neq \gamma'$.
For any $\nu \in \mathcal{P}$, we can obtain its index $\gamma(\nu)$ by placing each of its $n$
gaps into a set of $T$ bins.
\begin{align*}
 \gamma_i(\nu) &=
  \begin{cases}
  1    & \text{if }\;\; \Delta_i \in [\epsilon^{\frac{1}{T}}, 1]\\
   k       & \text{if }\;\; \Delta_i \in [\epsilon^{\frac{k}{T}}, \epsilon^{\frac{k - 1}{T}}) \text{ for some k } \in \{2,\ldots,T-1\} \\
   T    & \text{if }\;\; \Delta_i \in [0, \epsilon^{\frac{T-1}{T}}).
  \end{cases}
\end{align*}
We then define $\mathcal{P}_\gamma = \{\nu \in \mathcal{P}| \gamma(\nu) = \gamma\}$ to be
the set of distributions $\nu$ such that each arm gap $\Delta_i$ falls in the set of possible gap
values mapped to by $\gamma_i$.
}
\end{definition}

This partitioning has a partial order in the following sense:
if $\gamma, \gamma'\in[T]^n$ are such that $\gamma < \gamma'$ (elementwise),
then for any $\nu\in\Pcal_\gamma$ and $\nu'\in\Pcal_{\gamma'}$, 
$\Hcompl(\nu) < \Hcompl(\nu')$;
similarly, if $\gamma \leq \gamma'$, then
$\sup_{\nu\in\Pcal_{\gamma}} \Hcompl(\nu) \leq \sup_{\nu\in\Pcal_{\gamma'}} \Hcompl(\nu)$.
When the indices in $\gamma$ are large, the problems $\nu$ in $\Pcal_\gamma$ are harder.
The hardest partition, $\mathcal{P}_{T,\ldots,T}$,
contains problems with arms with means close to each other, while the easiest partition,
$\mathcal{P}_{1,\ldots,1}$, contains problems where all sub-optimal arms are far away
from $\mu_{[1]}$.

%% file: method.tex
\vspace{0.05in}
\section{Algorithm and upper bounds}
\vspace{0.05in}
\label{sec:algorithm}
\label{sec:method}

We now describe our algorithm for this setting.
In Algorithm~\ref{alg:fixed_delta_T},
we propose \algname{} (\algabbr),
a racing-style algorithm.
To describe it,
let us first define a few quantities.
Recall that
$N_{i,t}$ denotes the number of times arm $i$ is pulled on round $t$,
$\Ntilde_{i,t}$ is the number of times $i$ was pulled from rounds $1,\dots,t$,
and $X_{i,t,j}$ denotes the $j$-th reward of arm $i$ on round $t$.
Now define:
\begin{align*}
    &\widehat{\mu}_{i, t} := \frac{\sum_{k=1}^t\sum_{j=1}^{N_{i,k}} X_{i,k,j}}{\Ntilde_{i,t}},
    \hspace{0.60in}
    D(\tau, \delta) := \sigma \sqrt{\frac{(4 +
2\log(2))\log(nT/\delta)}{\tau}},
    \\
    &L_i(t, \delta) := \widehat{\mu}_{i,t} - D(\widetilde{N}_{i, t}, \delta),
    \hspace{0.57in}
    U_i(t, \delta) := \widehat{\mu}_{i,t} + D(\widetilde{N}_{i, t}, \delta).
\label{eqn:confintervals}\numberthis
\end{align*}
Here, $\widehat{\mu}_{i, t}$ is the empirical mean for arm $i$ at the end of round $t$.
$D(\tau, \delta)$ is a deviation function,
while $L_i(t, \delta)$ and $U_i(t,\delta)$  are lower and upper confidence bounds for arm $i$
after
round $t$.

Algorithm~\ref{alg:fixed_delta_T} proceeds in $T$ rounds and maintains a set of surviving arms based
on the above confidence intervals.
In round $t$, it pulls each surviving arm $N_t$ times and eliminates any arm whose
upper confidence bound lies below the lower confidence bound of any arm plus $\epsilon/\min(n,T)$.
If all of the confidence intervals trap the true means, we show that the algorithm can only reject
$\epsilon$-optimal arms if another $\epsilon$-optimal arm will remain in $S_{t+1}$.
This prevents the algorithm from
rejecting all $\epsilon$-optimal arms. At round $t$, the algorithm ensures that each surviving arm
has been
pulled at least $\bigOtilde(\epsilon^{-\frac{2t}{T}})$ times, resulting in a geometric increase in each
surviving arm's pull count over time.
This allows it to quickly allocate more samples to arms that survive longer, and therefore likely need more samples to distinguish, while avoiding over-committing to arms that can be eliminated with few samples.
The algorithm terminates either at the end of the $T$\ssth round or if at some point $|S_t|=1$,
at which point it returns the surviving
arm with the highest empirical mean as the recommendation
$\recrule$.



\begin{algorithm}[tb]
\caption{\algname{} (\algabbr{})}
\label{alg:fixed_delta_T}
\begin{algorithmic}
\STATE \textbf{Input:} Deadline $T$, error probability $\delta$, error tolerance $\epsilon$
\STATE $S_0 \leftarrow [n]$
\FOR {rounds $t = 1$ to $T$} 
        \STATE $N_{i,t} \leftarrow \left\lceil80\log\left(\frac{nT}{\delta}\right)\epsilon^{-\frac{2t}{T}}\right\rceil - \sum_{s=1}^{t-1} N_{i,s}$ for each arm $i \in S_t$.
        \STATE Pull each arm in $S_t$, $N_{i,t}$ times in parallel.
        \STATE $R_t \leftarrow \left\{i \in S_{t}\,|\; U_i(t, \delta) < \max_{j \in S_t} L_j(t,\delta) +
            \frac{\epsilon}{\min(n,T)}\right\}$.
    \COMMENT{See~\eqref{eqn:confintervals}}
        \IF{$|S_t| = 1$}
            \STATE \textbf{Return} arm in $S_t$
        \ENDIF
        \STATE $S_{t+1} \leftarrow S_t \setminus R_t$
\ENDFOR
\STATE \textbf{Return} $\arg\max_{i \in S_T} \widehat{\mu}_{i,T} $
\end{algorithmic}
\end{algorithm}

Our first result below bounds the number of samples required uniformly in each set of the partition.

\begin{theorem}\label{thm:partitioning_upper_bound}
The following is true for all $\gamma\in [T]^n$ and all $\nu \in \Pcal_\gamma$
(see Definition~\ref{defn:partitioning}) with probablity
at least $1-\delta$.
Algorithm~\ref{alg:fixed_delta_T} returns a recommendation $\recrule\in[n]$ where
$\mu_1-\mu_{\recrule} < \epsilon$ in time at most $T$ where the number of arm pulls $\widetilde{N}_T$ is at most
\begin{align*}
    \widetilde{N}_T \leq
    80\sigma^2\log\left(\frac{nT}{\delta}\right)
    \sum_{i=1}^n  \epsilon^{-\frac{2\gamma_i}{T}}
    &\leq 320\sigma^2\log\left(\frac{nT}{\delta}\right)\sup_{\nu \in \mathcal{P}_\gamma}\Hcompl(\nu).
\end{align*}
\end{theorem}


While in general, the above result provides the tightest bound on the number of samples required
by \algabbr{} on a problem $\nu$,
our next result establishes a straightforward relation
between the number of pulls $\widetilde{N}_T$ and $\Hcompl(\nu)$.
Recall the definition of $\Niopt$ from~\eqref{eqn:gapsdefn}.


\begin{theorem}\label{thm:fixed_b0}
The following is true for all $\nu\in\Pcal$ with probability at least $1-\delta$.
Algorithm~\ref{alg:fixed_delta_T} returns a recommendation $\recrule\in[n]$ where
$\mu_1-\mu_{\recrule} < \epsilon$ in at most $T$ rounds where the cost $\widetilde{N}_T$ is at most
\begin{align*}
    \widetilde{N}_T \;\leq
    \;  640\sigma^2\epsilon^\frac{-2}{T}\log\left(\frac{nT}{\delta}\right)\sum_{i=1}^n \Niopt
    \;
    \leq \; 640\sigma^2 \epsilon^\frac{-2}{T}\log\left(\frac{nT}{\delta}\right) \Hcompl(\nu).
\end{align*}
\end{theorem}

\begin{remark}
Observe that if $T \geq 2\log\left(1/\epsilon\right)$, then $\epsilon^{-\frac{2}{T}} \leq e$. In this
case, $\widetilde{N}_T \leq 640e\sigma^2\Hcompl(\nu)\log\left(\frac{nT}{\delta}\right)$.
While
$\epsilon^{-\frac{2}{T}}$ is the added cost due to the lack of opportunities to behave adaptively,
this cost is bounded by a constant factor when $T \geq 2\log\left(1/\epsilon\right)$.
\end{remark}





\paragraph{Comparison to prior work~\citep{jin2019efficient,agarwal2017learning}:}
~\citet{jin2019efficient} study the exact setting of this paper and
propose an algorithm with a cost
$O\Big(n\epsilon^{-2}\big(\log(1/\delta) +
\mathrm{ilog}_{1/\delta}^T(n)\big)\Big)$,
 where $\mathrm{ilog}^T(n)$ is defined recursively as follows:
$\mathrm{ilog}^T(n) = \log\big(\mathrm{ilog}^{T-1}(n)\big)\vee 1$ for $T > 0$, and $\mathrm{ilog}^0(n) = n$.
~\citet{agarwal2017learning} study exact best arm identification,
but when $\Delta_{[1]}$ is known.
Their upper bound has a similar flavor to the one in~\citep{jin2019efficient}, but with $\epsilon$
replaced with $\Delta_{[1]}$.
Therefore, we will focus on the above bound from~\citep{jin2019efficient}.
This bound, while capturing the worst case complexity, does not demonstrate adaptability
to problem instance. 
In fact, the passive strategy of pulling all arms $\bigOtilde(1/\epsilon^2)$ number of
times in a single round achieves the same bound as above, with the improvements mostly in
the lower order $\mathrm{ilog}_{1/\delta}^T(n)$ term.
The bounds in both works are not an artifact of their proof:
both algorithms pull each arm $\bigOtilde(\epsilon^{-2})$ times in the very first round and hence
the $\bigOtilde(n\epsilon^{-2})$ cost is unavoidable.

\paragraph{Proof sketch for Theorems~\ref{thm:partitioning_upper_bound} and~\ref{thm:fixed_b0}:}
We condition all of our upper bound analysis on the event that the confidence intervals always
trap the true mean for all arms on all rounds, which we show occurs with probability at
least $1 - \delta$.
Under this event, suboptimal arms will be correctly eliminated as long as
they are pulled sufficiently many times, which sufficiently shrinks the confidence intervals to
distinguish them from more promising arms.
This will be accomplished in at most $\bigOtilde\left((\epsilon +\Delta_{i})^{-2}\right)$ pulls,
which can be significantly less than $\bigOtilde\left(\epsilon^{-2}\right)$.
For the first result, we identify the hardest problem instance in $\Pcal_\gamma$ and show that if
the confidence intervals always trap the true means, the algorithm will not overpull any arms
relative to $\Niopt$.
For the second result, we show that because there are only limited opportunities for adaptive
behavior, there may be a problem instance where the algorithm will overpull arms in order to
eliminate them. However, since the rate of pulls increases by a factor of
$\epsilon^{-\frac{2}{T}}$ each round, this is the maximum factor any arm can be overpulled relative
to $\Niopt$.

%% file: lowerbound.tex
\vspace{-0.05in}
\section{Lower bounds}
\label{sec:lowerbound}
\vspace{-0.10in}


In this section, we state our hardness results. 
Our first result  provides a lower bound on the worst case complexity
in each set $\Pcal_\gamma$ of the partition. 
We provide a lower bound on the expected number of pulls and additionally, a high
probability lower bound when $n=2$.

\vspace{0.05in}
\begin{theorem}\label{thm:partitioning}
Let $n\geq 2$ and $\delta \leq 0.15$ be given.
Let $A$ be a $T$ round algorithm that is \emph{$(\epsilon,\delta)$-PAC}.
Then for all $\gamma \in [T]^n$,
     \[
    \inf_{A\in(\epsilon,\delta)\text{{\rm -PAC}}}
    \;\sup_{\nu\in\mathcal{P}_\gamma}\;
    \mathbb{E}_{A,\nu}\big[\widetilde{N}_T\big]
\geq
2\sigma^2\log\left(\frac{1}{2.4\delta}\right)\sup_{\nu\in\mathcal{P}_\gamma}\Hcompl(\nu)
\geq \frac{\sigma^2}{2}\log\left(\frac{1}{2.4\delta}\right)\sum_{i=1}^n
\epsilon^{-\frac{2\gamma_i}{T}}.
    \]
Moreover, when $n=2$,
for all $\gamma \in [T]^2$, there exists $\nu \in \Pcal_\gamma$ such that with probability
at least $1/6$,
\[
\widetilde{N}_T \geq
\frac{2\sigma^2}{3}\log\left(\frac{1}{2\delta}\right)\sup_{\nu\in\mathcal{P}_\gamma}\Hcompl(\nu)
\geq \frac{\sigma^2}{6}\log\left(\frac{1}{2\delta}\right)
\epsilon^{-\frac{2\gamma_i}{T}}.
\]
\end{theorem}
The above lower bound matches the upper bound in Theorem~\ref{thm:partitioning_upper_bound}
up to constant and logarithmic terms.
This shows that \algabbr{} is minimax optimal within each $\Pcal_\gamma$.
In the second claim, the probability can be made arbitrarily close to $1/2$
(with worse constants); we can also obtain a looser bound without the $\log(1/(2\delta))$
term with probability arbitrarily close to $1$.
We should emphasize that the upper bound and the
lower bounds are, strictly speaking, not comparable.
Theorem~\ref{thm:partitioning_upper_bound}, in addition to showing $(\epsilon,\delta)$-PAC
properties, also bounds the number of pulls with probability $1-\delta$.
In contrast, the above theorem lower bounds the number of pulls in expectation
or with constant probability.
This discrepancy between upper and lower bounds is not uncommon in the BAI
literature~\citep{kaufmann2016complexity,karnin2013almost,thananjeyan2021overcoming,jamieson2014lil}.

%
Next, while we are unable to provide a corresponding hardness result for 
Theorem~\ref{thm:fixed_b0}, we provide two partial results.
The first of these bounds the worst case ratio $\Ntilde_T/\Hcompl(\nu)$
when $T=2$ and $n=2$.

\vspace{0.05in}
\begin{theorem}\label{thm:lb_n2_T2}
Fix $n=2$.
Let $\epsilon < \frac{1}{10}$, $\delta < \frac{1}{32}$.
Let $A$ be any $(\epsilon, \delta)$-PAC algorithm for $T = 2$ rounds. Then, there exists
$\nu \in
\mathcal{P}$ such that with probability at least $\frac{1}{8}$,
the total number of pulls $\Ntilde_2$ satisfies,
\begin{align*}
    \frac{\widetilde{N}_2}{\Hcompl(\nu)} &
    \geq \frac{\sigma^2}{100}\epsilon^{-1}.
\end{align*}
\end{theorem}
If the ratio $\Ntilde_T/\Hcompl(\nu)$ is uniformly small across all problems in $\Pcal$,
it means that a $T$-round algorithm does not do
significantly worse than its sequential counterpart on any problem.
While Theorem~\ref{thm:fixed_b0} upper bounds this ratio by
$\tilde{O}\big(\epsilon^{-\frac{2}{T}}\big)$ on all problems, the lower bound shows that,
at least for this simple setting,
this ratio could be as large on some problems, up to constant and logarithmic factors.

We generalize the previous result for general $T$,
but for a restricted class of algorithms.
Let $\widetilde{\Acal}$ be the class of $(\epsilon,
\delta)$-PAC algorithms that run for $T$ rounds and choose values
$\{\{Q_{i,s}\}_{s\in[T]}\}_{i\in[n]}$ ahead of time.
Then, they \emph{adaptively} choose on which rounds to
pull each arm $i \in [n]$, but if it pulls arm $i$ on round $t$, it must pull it $Q_{i, t}$ times.
That is, $N_{i,t} \in \{0, Q_{i,t}\}$.
We have the following result.

\vspace{0.05in}
\begin{theorem}\label{thm:lb_restricted_class}
Fix $n=2$ and $T$.
Let $n= 2$ and let $\delta \leq \frac{1}{6}$, $\epsilon \leq 2^{-(T + 1)}$ be given
and let $A \in \widetilde{\alg}$.
Then, there exists $\nu \in \mathcal{P}$ such that with probability
at least $\frac{1}{6}$,
\begin{align*}
   \frac{\widetilde{N}_T}{\Hcompl(\nu)} &\geq \frac{\sigma^2}{192}\log\left(\frac{1}{2\delta}\right)\epsilon^{-\frac{2}{T}}.
\end{align*}

\end{theorem}
Once again, we find that this result matches Theorem~\ref{thm:fixed_b0}
up to constant and logarithmic factors.
Note that \algabbr{} is in the class $\widetilde{\Acal}$.
The two previous high probability lower bounds can be easily converted into results in terms of the
expected number of pulls $\mathbb{E}_{A,\nu}[\widetilde{N}_T]$, as in the following
corollary.

\vspace{0.05in}
\begin{corollary}[Corollary to Theorems~\ref{thm:lb_n2_T2} and~\ref{thm:lb_restricted_class}]
\label{corollary:expectation_lb}
Let $\mathcal{P}$ be the class of $2$-armed bandit models with $\sigma$ sub-Gaussian rewards. 
Then,
\textbf{(i)}
    Under the assumptions of Theorem~\ref{thm:lb_n2_T2},
    \begin{align*}
        \inf_{A\in(\epsilon,\delta)\text{{\rm -PAC}}}\;\sup_{\nu\in\mathcal{P}}\;\frac{\mathbb{E}_{A,\nu}\big[\widetilde{N}_2\big]}{\Hcompl(\nu)}
&\geq \frac{\sigma^2}{800}\epsilon^{-1}.
    \end{align*}
\textbf{(ii)}
    Under the assumptions of Theorem~\ref{thm:lb_restricted_class},
    \begin{align*}
    \inf_{A \in \widetilde{\alg}}\;
    \sup_{\nu \in \mathcal{P}}\; \frac{\mathbb{E}_{A, \nu}\big[\widetilde{N}_T\big]}{\Hcompl(\nu)} &\geq \frac{\sigma^2}{144}\log\left(\frac{1}{2\delta}\right)\epsilon^{-\frac{2}{T}}.
    \end{align*}
\end{corollary}

\paragraph{Comparison to prior work:}
~\citet{kalyanakrishnan2012pac}, who study the sequential setting, establish a 
$\bigOmegatilde(n/\epsilon^2)$ worst case complexity for $(\epsilon,\delta)$-PAC BAI.
\citet{agarwal2017learning} provide a
$\bigOmega\big(\frac{n}{\Delta_{[1]}^2 T^4} \mathrm{ilog}^T(n)\big)$ lower bound on the
worst case complexity for their problem of finding the best arm with known $\Delta_{[1]}$
in $T$ rounds.
However, this bound is better than $n/\Delta_{[1]}^2$ only for $T$ such that
$T^4\in \littleO({\rm ilog}(T))$, which severely limits its applicability.
Crucially, these results do not capture the main advantage
adaptivity has to offer: the ability to adapt to problem difficulty.
While our lower bounds also consider the the worst case over a class of problems,
we do so either over smaller classes depending on $T$, 
or study the worst case ratio relative to the problem's complexity term $\Hcompl(\nu)$.

\vspace{-0.05in}
\subsection{Proof sketches for Theorems~\ref{thm:partitioning},~\ref{thm:lb_n2_T2},
and \ref{thm:lb_restricted_class}}
\label{sec:lb_proof_sketches}

An important ingredient in most lower bound analyses is a change of measure lemma.
We use one provided in~\citet{kaufmann2016complexity} (While their lemma was given for the
sequential setting, it is straightforward to establish a similar result for $T<\infty$).
However, as this lemma only allows us to upper bound the expected number of arm pulls,
it is not sufficient for our purposes.
In particular, the proofs of Theorems~\ref{thm:lb_n2_T2} and~\ref{thm:lb_restricted_class}
rely on showing that the number of pulls will be large for some arms with constant probability.
Therefore, we first show that on a two-armed problem $\Ntilde_{i,T} \geq
\bigOmegatilde((\Delta_i+\epsilon)^{-2})$ with constant probability.
Our proof, which uses the change of measure lemma, shows by contradiction that if the number of
pulls is small, an alternative algorithm which does not execute as many arm pulls will do well
on the problem.
We will refer to this result as HPCM (high probability change of measure).
Unfortunately, an HPCM result for general $n$ appears difficult and hence
Theorems~\ref{thm:lb_n2_T2} and~\ref{thm:lb_restricted_class} are stated for 2-armed problems.
A result for general $n$ can be obtained  using the same intuitions we outline below, but with
an HPCM for arbitrary $n$.

The proof of Theorem~\ref{thm:partitioning} simply applies the change of measure
lemma and HPCM to the hardest problem in each
subclass, i.e., when all gaps are as small as possible.
The main novelty in this proof, when taken in conjunction with
Theorem~\ref{thm:partitioning_upper_bound}, is in the design of the partitioning
(Definition~\ref{defn:partitioning}).
For Theorem~\ref{thm:lb_restricted_class}, we use HPCM to first argue that any algorithm
in $\Acal$ must have $\sum_{t}Q_{i,t}\in \bigOmegatilde(\epsilon^{-2})$ for all
$i\in[2]$. 
Then, we show that any sequence of $Q_{i,t}$ values which satisfy this constraint will have
a gap of at least $\bigOmegatilde(\epsilon^{-2/T})$ between the possible values that can be generated by summing any subset of these values.
A problem $\nu$ with difficulty falling in this gap will pay this additional cost, as the only way to
pull arms sufficiently to be $(\epsilon, \delta)$-PAC will be to overpull arms by at least this amount.

For Theorem~\ref{thm:lb_n2_T2}, 
we use HPCM to show that when the arms are exactly
$\epsilon$ away, $\bigOmega(\epsilon^{-2})$ pulls are necessary.
We then use the high probability version of Pinsker's inequality to show that if an arm only pulls
$\bigO(\epsilon^{-1})$ times on the first round, then with constant probability,
 it may not be able to distinguish
between a hard problem  where a sub-optimal arm is $\epsilon$ away, and an easier
problem where the sub-optimal arm is
$\sqrt{\epsilon}$ away.
Under this event, it has to pull at least $\bigOmega(\epsilon^{-2})$ times.
Therefore, if an algorithm pulls less than $\bigO(\epsilon^{-1})$ times, then it has large
$\Ntilde_T/\Hcompl(\nu)$ ratio
on problems where the sub-optimal arm is $\sqrt{\epsilon}$ away.
If however, it pulls more than $\bigO(\epsilon^{-1})$ times in the
first round, it has a poor ratio on problems where $\Delta_{[1]}$ is very
large, so that $\Hcompl(\nu)$ is a constant.

%% file: experiments.tex
\vspace{-0.05in}
\section{Simulations}
\vspace{-0.05in}
\label{sec:experiments}

We evaluate \algabbr{} on a set of simulation experiments against a set of
baselines. The purpose of these experiments is to study whether \algname{} is able to effectively
reduce its cost as the number of rounds or error tolerance are increased, as suggested by the
theoretical results.

\begin{figure}[t!]
    \centering
    \begin{subfigure}
        \centering
        \includegraphics[width=2.2in]{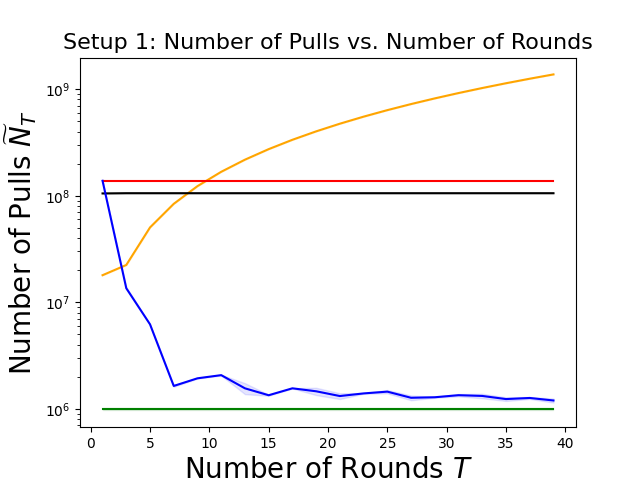}
    \end{subfigure}%
    ~ 
    \begin{subfigure}
        \centering
        \includegraphics[width=2.2in]{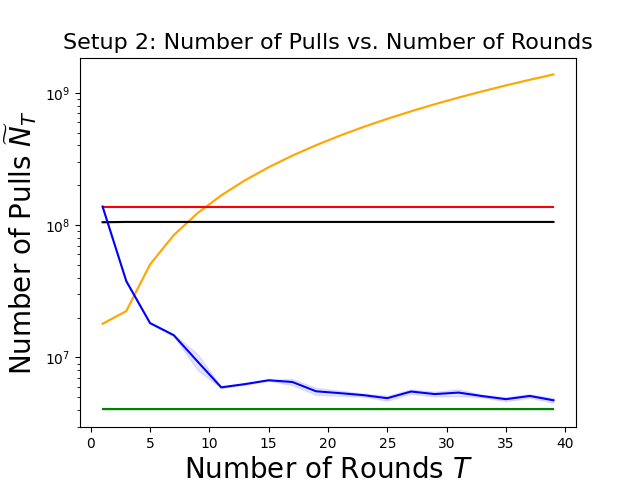}
    \end{subfigure}
    \begin{subfigure}
        \centering
        \includegraphics[width=2.2in]{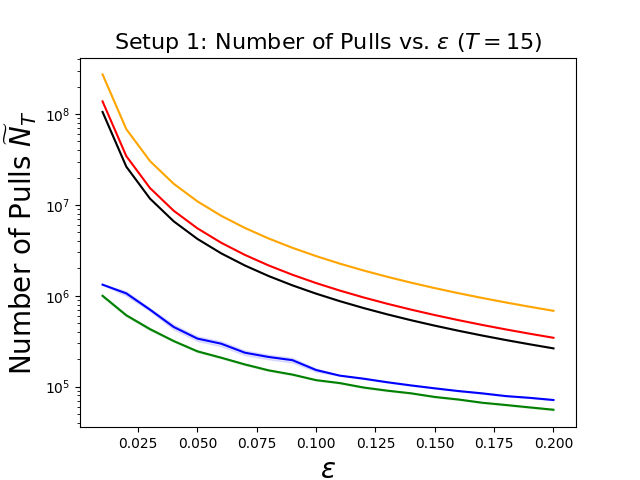}
    \end{subfigure}
    \begin{subfigure}
        \centering
        \includegraphics[width=2.2in]{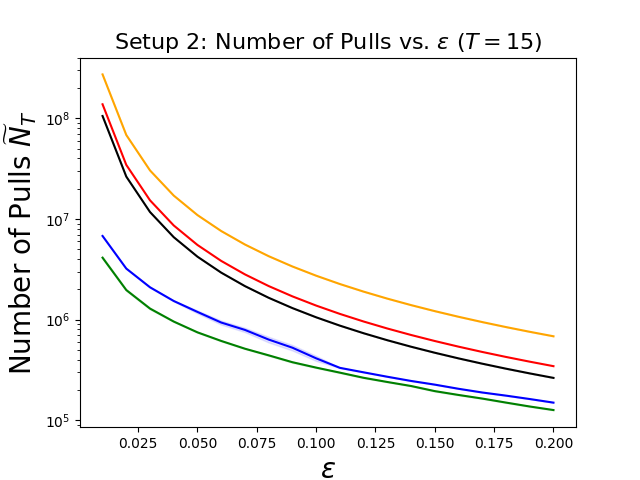}
    \end{subfigure}
    \begin{subfigure}
        \centering
        \includegraphics[width=3.5in]{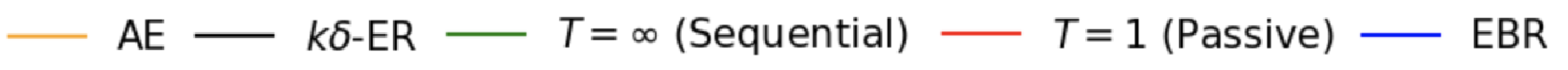}
    \end{subfigure}
    \caption{\textbf{Simulation experiments:}
In the top row, we have shown how the number of pulls varies with the deadline $T$
when $\epsilon = 0.01$ is fixed on the two setups.
In the bottom row we show how the number of pulls varies with the error tolerance $\epsilon$
when $T=15$ is fixed.
Each point on the curves was obtained by averaging over 100 runs and error bars indicate
one standard error.
}
\label{fig:experiments}
\vspace{-0.10in}
\end{figure}

\textbf{Baselines:}
We compare \algname{} to the following baseline $(\epsilon, \delta)$-PAC algorithms:
\vspace{-0.05in}
\begin{enumerate}[leftmargin=0.4in]
    \item \textbf{Top-$k$ $\delta$-Elimination with Limited Rounds ($k-\delta$ER):} an elimination-style algorithm proposed for this setting in~\cite{jin2019efficient}.
\vspace{-0.02in}
    \item \textbf{Aggressive Elimination(AE)}: an elimination-style algorithm proposed by~\citet{agarwal2017learning} in a
similar setting, but where $\Delta_{[1]}$ is known. We provide the algorithm $\epsilon$ instead
of $\Delta_{[1]}$, which is typically difficult to know in practice.
\vspace{-0.02in}
    \item \textbf{$T = \infty$ (Sequential):} a sequential algorithm that pulls arms
sequentially, and has elimination condition as \algabbr{}. We use the
deviation function from~\citet{jun2016top} to construct confidence bounds, which does not
increase with $T$, but is initially more conservative than ours. Naively setting $T = \infty$ with
the deviation function in \algabbr{}, results in unbounded confidence intervals.
    \item \textbf{$T = 1$ (Passive):} a passive algorithm that pulls all arms
    $\bar{N}_\epsilon(\delta) = \frac{80\sigma^2}{\epsilon^2}\log\left(\frac{n}{\delta}\right)$
times in a single round, then outputs the arm with the highest empirical mean
    As per our confidence intervals, $\bar{N}_\epsilon(\delta)$ samples are sufficient to
     identify an $\epsilon$-optimal arm in the passive setting.
\end{enumerate}

\textbf{Experimental setup:}
We provide two types of experiments.
In the first, we fix $\epsilon = 0.01$, and vary the deadline $T$.
In the second, we fix $T = 15$ and vary $\epsilon$.  
In all experiments, we use $\delta=0.01$.
We consider two bandit models, with $n = 100$ arms and Bernoulli rewards, designed as follows.
\begin{enumerate}[leftmargin=0.4in]
\vspace{-0.05in}
    \item Setup 1 (evenly spaced arms): The arm means form an arithmetic sequence from $0.1$ to $0.9$.
\vspace{-0.03in}
    \item Setup 2 (all arms are close): The arm means form an arithmetic sequence from $0.65$ to $0.9$.
\end{enumerate}
The experimental setups subject the algorithms to a variety of conditions, where all, some, or a few
arms will need to be pulled many times before a decision can be made with high confidence.

\textbf{Results:}
While the algorithms are only required to identify an $\epsilon$-optimal arm with probability at
least $1 - \delta$, we find that they all identify such an arm in every trial.
This is consistent with
prior work which suggest that these confidence intervals can be conservative in
practice~\cite{thananjeyan2021overcoming,kandasamy2015high,wang2017batched}.
They can be tuned for better empirical performance, but we will not
delve into this in this work.
Therefore, we do not report success rate, and instead
focus on the cost, i.e the number of samples, in Figure~\ref{fig:experiments}.
We find that \algabbr{} is able to outperform baselines on all tasks, except
for the sequential oracle algorithm, which has infinite time budget. In the sweep over number of
rounds, \algabbr{} quickly reduces the number of pulls it requires by several orders of magnitude
when given additional time.
In the sweep over $\epsilon$, \algabbr{} consistently
stays within an order of magnitude of the sequential algorithm and is at least an order of magnitude
more efficient than the passive \algabbr{} and $k-\delta$ER.
In our experiments, we found that AE does not eliminate arms until the
last few rounds, after which it aggressively does so; therefore,
increasing its number of rounds actually
increases its number of pulls. While $k\delta$-ER aggressively eliminates arms immediately, the number of pulls is dominated by the first round,
where $\tilde{O}\left(\frac{n}{\epsilon^2}\right)$ pulls are executed. So, even as $T$ is increased,
it does not perform much better.


%% file: conclusion.tex
\vspace{-0.05in}
\section{Conclusion}\label{sec:conclusion}
\vspace{-0.05in}

We study PAC BAI under a deadline where a decision-maker must
identify an $\epsilon$-optimal arm while minimizing the number of pulls, or cost, to do so.
When the deadline is short, 
the ability to behave adaptively is limited;
our upper and lower bounds tightly quantify this phenomenon.
Our proposed algorithm, \algabbr{}, has several optimality properties and outperforms
baselines in our simulations.
One
avenue for future work is to generalize Theorem~\ref{thm:lb_n2_T2} to work for general $n$ and $T$.
Another interesting area of future work is studying BAI with a fixed cost and deadline $T$
(instead of fixing $\delta$
and $T$ as in this paper), and minimizing the probability of mis-identifying
the best arm,
which is a finite round variation of the sequential fixed budget problem.

%% file: appendix.tex
\section{Proof of Theorems~\ref{thm:partitioning_upper_bound} and \ref{thm:fixed_b0}}
In this section, we will prove the upper bounds on cost of \algabbr{} in
Theorems~\ref{thm:partitioning_upper_bound} and~\ref{thm:fixed_b0}. We will first introduce a few
preliminary results.

\subsection{Correctness of confidence intervals}
Denote,
$\mathcal{E}_i(\delta) = \{\forall t\in[T], L_i(t,\delta) \leq \mu_i \leq U_i(t,\delta)\}$, which is
the event that the confidence bounds capture the true mean $\mu_i$ of arm $i$ at all rounds of the
algorithm. Let $\mathcal{E} = \cap_{i\in [n]} \mathcal{E}_i(\delta)$ be the event that this is true
for all arms. We will first show that $\mathcal{E}$ occurs with probability at least $1 - \delta$.

We will use the following lemma to show that the confidence intervals
in Algorithm~\ref{alg:fixed_delta_T} trap the true means at all rounds of the algorithm.

\begin{lemma}\label{lemma:confidence}
Define $\rho = \sigma^2(4 + 2\log(2))$ and assume $\log\left(\frac{\delta}{nT}\right) \geq 2$. Let $\mathcal{P}$ be the set of $n$ $\sigma$ sub-Gaussian random variables, and let $\nu \in \mathcal{P}$.
If Algorithm~\ref{alg:fixed_delta_T} is run on $\nu$, for any $t \in [T]$ and any
$i\in[n]$,
\begin{align*}
    &\mathbb{P}_{A,\nu}\left(\mu_i \in \left(\hat{\mu}_{i, t} \pm \sqrt{\frac{\rho \log(\frac{nT}{\delta})}{\widetilde{N}_{i,t}}}\right)\right) \geq 1 - \frac{\delta}{nT}
\end{align*}
\end{lemma}

To prove this lemma, we require the following result from~\cite{de2004self}.

\begin{lemma}\label{lemma:concentration}
(\citet{de2004self}, Corollary 2.2) Let $A$, $B$ be random variables such that $A \geq 0$ a.s. and $\mathbb{E}\left[e^{\nu B - \frac{\nu^2 A^2}{2}}\right] \leq 1$ for all $\nu \in \mathbb{R}$. Then, $\forall c \geq 2$, $\mathbb{P}\left(|B| > cA\sqrt{2 + \log(2)}\right) \leq e^{-\frac{c^2}{2}}$.
\end{lemma}

\textit{{Proof of Lemma~\ref{lemma:confidence}:}}
We want to show that $\mathbb{P}_{A,\nu}\left(\mu_i \in \left(\hat{\mu}_{i, t} \pm \sqrt{\frac{\rho \log(\frac{nT}{\delta})}{\widetilde{N}_{i,t}}}\right)\right) \geq 1 - \frac{\delta}{nT}$ by applying Lemma~\ref{lemma:concentration}. However, we will need to find random variables $A$ and $B$ that satisfy its preconditions such that the resulting inequality in the lemma's implication proves our result.
For this, define 
\begingroup
\allowdisplaybreaks
\begin{align*}
    B
    &= \sum_{s = 1}^{N_{i, 1}} (X_{i,1,s} - \mu_i) + \sum_{s = 1}^{N_{i, 2}} (X_{i, 2, s} - \mu_i) + \cdots + \sum_{s = 1}^{N_{i, t}} (X_{i,t, s} - \mu_i)\\
    A^2 &= \sum_{s = 1}^{\widetilde{N}_{i,t}} \sigma^2 = \sigma^2 \widetilde{N}_{i,t}
    = \sum_{s = 1}^{N_{i, 1}} \sigma^2 + \sum_{s = 1}^{N_{i, 2}} \sigma^2 + \cdots + \sum_{s = 1}^{N_{i, t}} \sigma^2
\end{align*}
\endgroup
Now, expand
\begin{align*}
    \nu B - \frac{\nu^2A^2}{2} &= \sum_{j = 1}^t \left(\nu\sum_{s = 1}^{N_{i, j}}(X_{i, j, s} - \mu_i) - \frac{\sigma^2\nu^2}{2}N_{i, j}\right)\\
    &= \sum_{j = 1}^t Q_{j}
\end{align*}
where $Q_j = \left(\nu\sum_{s = 1}^{N_{i, j}}(X_{i, j, s} - \mu_i) - \frac{\sigma^2\nu^2}{2}\widetilde{N}_{i, j}\right)$.
Recall that $\mathcal{F}_{t-1}$ is the $\sigma$-algebra generated by the observations up to time $t$. Observe that
\begin{align*}
    \mathbb{E}_{A,\nu}\left[e^{\sum_{j = 1}^t Q_j}\right] &= \mathbb{E}_{A,\nu}\left[e^{\sum_{j = 1}^{t-1} Q_j}\mathbb{E}_{A,\nu}\left[e^{Q_t}| \mathcal{F}_{t - 1}\right]\right]\\
    &= \mathbb{E}_{A,\nu}\left[e^{\sum_{j = 1}^{t-1} Q_j}\mathbb{E}_{A,\nu}\left[e^{\left(\nu\sum_{s = 1}^{N_{i, t}}(X_{i, j, s} - \mu_i) - \frac{\sigma^2\nu^2}{2}N_{i, t}\right)}| \mathcal{F}_{t - 1}\right]\right]\\
    &= e^{-\frac{\sigma^2\nu^2}{2}N_{i, t}}\mathbb{E}_{A,\nu}\left[e^{\sum_{j = 1}^{t-1} Q_j}\mathbb{E}_{A,\nu}\left[e^{\left(\nu\sum_{s = 1}^{N_{i, t}}(X_{i,j, s} - \mu_i) \right)}| \mathcal{F}_{t - 1}\right]\right]\\
    &\leq e^{\frac{\sigma^2\nu^2}{2}N_{i, t}-\frac{\sigma^2\nu^2}{2}N_{i, t}}\mathbb{E}_{A,\nu}\left[e^{\sum_{j = 1}^{t-1} Q_j}\right] = \mathbb{E}_{A,\nu}\left[e^{\sum_{j = 1}^{t-1} Q_j}\right]
\end{align*}
Repeatedly expanding the last term in the sum as above, we have that:
\begin{align*}
    \mathbb{E}_{A,\nu}\left[e^{\sum_{j = 1}^t Q_j}\right] &\leq 1.
\end{align*}
We can now apply Lemma~\ref{lemma:concentration}, which states that:
\begin{align*}
    \mathbb{P}_{A,\nu}\left(\left|\sum_{j=1}^t\sum_{s = 1}^{N_{i, j}}(X_{i,j,s} - \mu_i)\right| \geq c\sqrt{\widetilde{N}_{i,t}}\sigma \sqrt{2 + \log(2)}\right) \leq e^{-\frac{c^2}{2}}
\end{align*}
Rearranging, we have that
\begin{align*}
    \mathbb{P}_{A,\nu}\left(\left|\hat{\mu}_i - \mu_i)\right| \geq \frac{c\sigma\sqrt{2 + \log(2)}}{\sqrt{\widetilde{N}_{i,t}}} \right) \leq e^{-\frac{c^2}{2}}
\end{align*}
Setting $e^{-\frac{c^2}{2}} = \frac{\delta}{nT}$ and solving for $c$, we have that:
\begin{align*}
    \mathbb{P}_{A,\nu}\left(\left|\hat{\mu}_i - \mu_i)\right| \geq \sqrt{\frac{\sigma^2(4 + 2\log(2))\log\left(\frac{nT}{\delta}\right)}{\widetilde{N}_{i,t}}} \right) \leq \frac{\delta}{nT}
\end{align*}
proving the claim.\qed




By Lemma~\ref{lemma:confidence} and the union bound,  $\mathbb{P}_{A,\nu}\left[\cap_{i=1}^n \mathcal{E}_i(\delta)\right] \geq 1 - \delta$. Hereafter, we will assume $\mathcal{E}$ and show that the algorithm always outputs an $\epsilon$-optimal arm in this event and bound its cost.

\subsection{Arm elimination correctness}
We will show now that, conditioned on $\mathcal{E}$, if \algabbr{} eliminates an arm, another surviving arm has mean close to it.

\begin{lemma}\label{lemma:survival}
Assume $\mathcal{E}$ and $\log\left(\frac{\delta}{nT}\right) \geq 2$. In Algorithm~\ref{alg:fixed_delta_T}, define $j_{\rm max, t} = \arg\max_{j\in S_t} L_j(t,\delta) + \eta$. Then
\begin{align}
    &U_i(t,\delta) < L_{j_{\rm max, t}}(t,\delta) + \eta \implies \mu_i - \mu_{j_{\rm max, t}} \leq \eta
\end{align}
\end{lemma}

\begin{proof}
By conditioning on $\mathcal{E}$, we know that $\mu_i \leq U_i(t, \delta)$ and $\mu_{j_{\rm max, t}} \geq L_{j_{\rm max, t}}(t, \delta)$. So, $\mu_i - \mu_{j_{\rm max, t}} \leq U_i(t, \delta) - L_{j_{\rm max}}(t, \delta) \leq \eta$.
\end{proof}

\begin{lemma}\label{lemma:alg_survival}
Assume $\mathcal{E}$ and $\log\left(\frac{\delta}{nT}\right) \geq 2$. In
Algorithm~\ref{alg:fixed_delta_T}, $S_T$ contains an $\epsilon$-optimal arm.
\end{lemma}

\proof{
At time $t$, let $\mu_{\rm max, t}$ be the highest surviving arm mean. By Lemma~\ref{lemma:survival}, if an arm is eliminated at time $t$, then $\mu_{\rm max, t+1} \geq \mu_{\rm max, t} - \frac{\epsilon}{n \wedge T}$. If an arm is not eliminated, $\mu_{\rm max, t+1} = \mu_{\rm max, t}$. Since only $n$ arms can be eliminated, $\mu_{\rm max, T} \geq \mu_{[1]} - \epsilon$ if $T \geq n$, since $\mu_{\rm max, 1} = \mu_{[1]}$. If $T < n$, then by a similar argument, each round will let an arm that is within $\epsilon/T$ of the previously best surviving arm survive. Since there are only $T$ rounds, an $\epsilon$-optimal arm will survive all rounds.
\qed
}

So, conditioned on $\mathcal{E}$, \algabbr{} will not eliminate the only remaining $\epsilon$-optimal arm. Let us now focus on how many pulls are required to eliminate all $\epsilon$-suboptimal arms.

\begin{remark}
The additional $+ \epsilon/(n\wedge T)$ term in the rejection condition of \algabbr{} does not affect correctness and overall cost of the algorithm. It is, however, a small optimization that preserves these properties while reducing cost in practice.
\end{remark}

\subsection{Arm elimination cost}
Conditioned on $\mathcal{E}$, we introduce the following result that guarantees elimination of arms after a gap-dependent number of pulls.

\begin{lemma}\label{lemma:runtime}
Assume $\mathcal{E}$, $\log\left(\frac{\delta}{nT}\right) \geq 2$, and let
$\bar{N}_i := \frac{80\sigma^2\log\left(\frac{nT}{\delta}\right)}{\Delta_i^2}$. In Algorithm~\ref{alg:fixed_delta_T}, let $N'(t) = \min_{i\in S_r} N_i(r)$. Then,
\begin{align}
    &\forall r, i, \left(N'(t)\geq\bar{N}_i,\ \mu_i \leq \mu_{[1]} - \epsilon \implies U_i(r,\delta) < \max_{j\in S_r} L_j(r,\delta) + \epsilon/(n\wedge T)\right)\label{eq:ub_elim_time}
\end{align}
\end{lemma}
So, as long as $\mathbb{P}(\cap_{i=1}^n \mathcal{E}_i) \geq 1 - \delta$, the algorithm will output the correct set of arms after each surviving arm has been pulled $\bar{N}_i$ times with probability at least $1 - \delta$. This lemma is similar to Lemma 2 from~\citet{jun2016top}, but with a modification to $\bar{N}_i$ and the elimination condition to incorporate the $\epsilon$ error tolerance in this setting in addition to more aggressive confidence bounds.

\begin{proof}
For brevity, let $L_i(t)$ and $U_i(t)$ denote $L_i(t,\delta)$ and $U_i(t,\delta)$. We will show this result in the case that no arms have been eliminated yet. Generalizing to the case where arms have been eliminated is equivalent to showing the claim for an alternate problem where we ignore the arms that have already been eliminated from contention and reuse all samples from the old problem to eliminate the next easiest arm.

Let us start with the implication in~\ref{eq:ub_elim_time}. Let $\hat{\mu}_{\hat{1}, t}$ denote the arm with the highest empirical mean at round $t$. Assume that the RHS is false: $U_i(t,\delta) \geq \max_{j\in S_t} L_j(t,\delta) + \epsilon/(n\wedge T)$. Because $D(\widetilde{N}_{i, t},\delta) \leq D(N'(t),\delta)$,
\begin{align*}
    U_i(t) &\leq \hat{\mu}_{i,t} + D(N'(t),\delta) \leq \mu_i + 2D(N'(t),\delta),
\end{align*}
and 
\begin{align*}
    U_i(t) &\geq \max_{j\in S_t} L_j(t) + \epsilon/(n\wedge T) = L_{\hat{1}}(t) +  \epsilon/(n\wedge T) \geq \hat{\mu}_{\hat{1}, t} - D(N'(t),\delta) + \epsilon/(n\wedge T),
\end{align*}
\begin{align*}
    \implies \mu_i + 2D(N'(t),\delta) \geq \hat{\mu}_{\hat{1},t} - D(N'(t),\delta) +  \epsilon/(n\wedge T).
\end{align*}
Because $\hat{\mu}_{\hat{1},t} \geq \mu_k - D(N'(t),\delta)$ under $\mathcal{E}$,
\begin{align*}
    \mu_i + 2D(N'(t),\delta) \geq \mu_1 - 2D(N'(t),\delta) + \epsilon/n\\
    \implies \Delta_i + \epsilon/(n\wedge T) \leq 4D(N'(t),\delta)
\end{align*}
Rearranging, we then have that
\begin{align*}
    N'(t) &\leq \left\lfloor \frac{80\sigma^2\log\left(\frac{2nT}{\delta}\right)}{(\Delta_i + \epsilon/(n\wedge T))^2} \right\rfloor\\
    &<  \frac{80\sigma^2\log\left(\frac{nT}{\delta}\right)}{\Delta_i^2}\numberthis\label{eq:real_bound}\\
\end{align*}
This part of the lemma states, that as long as all surviving arms have been pulled sufficiently, a suboptimal arm $i$ can be eliminated correctly using the elimination rule in \algabbr{}. Since the best possible arm that must be eliminated has gap $\epsilon$, the \algabbr{} must pull arms at most $\frac{80\log\left(\frac{nT}{\delta}\right)}{\epsilon^2} \leq \bar{N}_{\epsilon}(\delta)$ times.
\end{proof}

The above lemma suggests that all suboptimal arms will be eliminated after pulling at most $\bar{N}_{\epsilon}(\delta)$ times, leaving only $\epsilon$-optimal arms, which will exist by Lemma~\ref{lemma:alg_survival}. So, \algabbr{} outputs an $\epsilon$-optimal arm with probability at least $1 - \delta$. However, some arms can be eliminated after pulling them $\bar{N}_i$ times, which may be significantly less than $\bar{N}_{\epsilon}(\delta)$. We must now bound the cost of the algorithm to show that it performs well on a wide range of problem instances.



\subsection{Proof of Theorem~\ref{thm:partitioning_upper_bound}}
\proof{
Let $\nu \in \mathcal{P}_{\gamma}$. From the previous results, we know that the confidence bounds of \algabbr{} capture the true means with probability at least $1 - \delta$ (Lemma~\ref{lemma:confidence}). We will again condition on the event that the confidence bounds capture the true means. Conditioned on this event, we also know that if suboptimal, the $i$-th arm can be correctly eliminated after pulling all remaining arms $\frac{80\log\left(\frac{nT}{\delta}\right)}{\Delta_i^2}$ times (Eq.~\ref{eq:real_bound}). Since $\nu \in \mathcal{P}_{\gamma}$,
\begin{align*}
    \frac{80\sigma^2\log\left(\frac{nT}{\delta}\right)}{\Delta_i^2}
    &\leq \frac{80\sigma^2\log\left(\frac{nT}{\delta}\right)}{\epsilon^{\frac{2\gamma_i}{T}}}.
\end{align*}
We know that all surviving arms at the end of round $\gamma_i$, have been pulled at least this many times.
So, with probability at least $1 - \delta$, \algabbr{} identifies an $\epsilon$-optimal arm with at most
\begin{align*}
    \widetilde{N}_T
    &\leq \sum_{i = 1}^n \frac{80\sigma^2\log\left(\frac{nT}{\delta}\right)}{\epsilon^{\frac{2\gamma_i}{T}}} \\
    &\leq 320\sigma^2 \Hcompl(\nu_\gamma)\log\left(\frac{nT}{\delta}\right)
\end{align*}
where $\nu_\gamma \in \mathcal{P}_\gamma$ has $\Delta_i = \epsilon^\frac{\gamma_i}{T}$.
\qed}

\subsection{Proof of Theorem~\ref{thm:fixed_b0}}
We will now use the above results to prove Theorem~\ref{thm:fixed_b0}. We will condition on $\mathcal{E}$, which occurs with probability at least $1 - \delta$ (Lemma~\ref{lemma:confidence}). By the previous discussion, we know the algorithm outputs an $\epsilon$-optimal arm with probability at least $1 - \delta$. However, we still need to bound its cost.


Here, define $\bar{N}_i := 320\sigma^2\log\left(\frac{nT}{\delta}\right)\Niopt$, $\beta =
\epsilon^{-\frac{1}{T}}$, and $B_0 = 80\sigma^2\log\left(\frac{nT}{\delta}\right)$. Suppose a
suboptimal arm $i$ is pulled at least $\bar{N}_i$ times after $k$ rounds and fewer than $\bar{N}_i$
times after $k-1$ rounds. By Lemma~\ref{lemma:runtime}, we will eliminate arm $i$ after the $k$-th
round at the latest. Then, it was pulled $\sum_{t=0}^k \widetilde{N}_{i,t}$ times, and we have that:




\begin{align*}
    \beta^{k-1} &\leq \bar{N_i}/B_0 \leq 2\beta^{k}
    \implies \left\lceil\beta^kB_0\right\rceil \leq 2\beta\bar{N}_i
\end{align*}
So, the algorithm can overshoot by a factor of at most $2\beta$ in each round, assuming it did not overshoot in the first round. If it overshot in the first round, it did so by a factor of at most $B_0$. The total cost again can be computed by upper bounding the number of pulls for each arm and summing them up:
\begin{align*}
    \widetilde{N}_T &\leq \sum_{\left\{i:\Delta_i \geq \epsilon\right\}}^n 2\beta\bar{N}_i \vee B_0 + \sum_{\left\{i:\Delta_i < \epsilon\right\}}^n \bar{N}_{\epsilon}(\delta)\\
    &= 640\sigma^2\log\left(\frac{nT}{\delta}\right)\epsilon^{-\frac{1}{T}}\sum_{\left\{i:\Delta_i \geq \epsilon\right\}} \bar{N}_i + \sum_{\left\{i:\Delta_i < \epsilon\right\}} \bar{N}_{\epsilon}(\delta)\\
    &\leq 640\sigma^2\log\left(\frac{nT}{\delta}\right)\epsilon^{-\frac{1}{T}}\Hcompl(\nu)
\end{align*}
\qed

%% file: app_lower_bound.tex
\section{Proof of Theorems~\ref{thm:partitioning},~\ref{thm:lb_n2_T2},
and~\ref{thm:lb_restricted_class}}

In this section, we will prove the lower bounds for this problem setting in
Section~\ref{sec:lowerbound}.
%
First, let us begin by refreshing some notation. Recall that at round $t$, algorithm $A$ takes
action $A_t = \left\{N_{i,t}\right\}_{i=1}^n \in \mathbb{N}^n$. Let $\widetilde{N}_t = \sum_{i=1}^n
N_{i,t}$ denote the total number of pulls at round $t$. We will assume that the samples for each arm
$i$ are generated an infinite number of times at each round. When action $A_t$ is executed, we will
observe the first $N_{i,t}$ samples from each arm $i$.
In addition, recall that denote the samples observed at round $t$.

Let us write down the log-likelihood ratio at round $t$ between bandit models $\nu$ and $\nu'$, which are absolutely continuous wrt each other and have densities $(f_1,\ldots,f_n)$ and $(f_1',\ldots,f_n')$ respectively:
\begin{align*}
    L_t &= L\left(\left\{A_s\right\}_{s=1}^t, \left\{O_t\right\}_{s=1}^t\right) = L(A_{1:t}, Z_{1:t})
    = \sum_{i=1}^n\sum_{s=1}^t\sum_{k=1}^{N_{i,s}}\log\left(\frac{f_i(X_{i, s, k})}{f_i'(X_{i,s,k})}\right)
\end{align*}

In this section, we will denote the binary relative entropy as $d(x, y) = x \log(x/y) + (1 - x)\log((1-x)/(1-y))$ and we define $d(0, 0) = d(1, 1) = 0$.

The following change of measure lemma will be useful in our analysis.
It is based off of Lemma 1 in~\citet{kaufmann2016complexity} who prove an identical result for the
sequential setting.
Its proof, which also uses very similar techniques to~\citet{kaufmann2016complexity},
 is given at the end in Appendix~\ref{sec:comlemma}.

\begin{lemma}\label{lemma:exp_kl_bound}
Let $A$ be an algorithm and let $\nu, \nu'$ be two bandit models from $\mathcal{P}$ with $n$ arms
s.t. $\forall i \in [n]$, $\nu_i, \nu_i'$ are absolutely continuous w.r.t. each other.
Then, for all $t\in[T]$, we have,
\begin{align*}
    \sum_{i\in[n]}\mathbb{E}_\nu\big[\widetilde{N}_{i,t}\big]D_{KL}(\nu_i, \nu_i') &\geq
\sup_{\mathcal{E}\in\mathcal{F}_t} d(\mathbb{P}_\nu(\event), \mathbb{P}_{\nu'}(\event)).
\end{align*}
where $d(x, y) = x \log(x/y) + (1 - x)\log((1-x)/(1-y))$ and we define $d(0, 0) = d(1, 1) = 0$.
\label{lem:comlemma}
\end{lemma}

In the next lemma, we will require the following fact:
\begin{fact}
Let $x,y\in[0,1]$ and $d(x, y) = x \log(x/y) + (1 - x)\log((1-x)/(1-y))$. We have that
\begin{align*}
    d\left(\frac{2}{3}, x\right) &= \frac{2}{3}\log\left(\frac{2}{3x}\right) + \frac{1}{3}\log\left(\frac{1}{3(1-x)}\right) \geq \frac{1}{3}\log\left(\frac{1}{2x}\right).
\end{align*}
\end{fact}
\proof{
To see this, subtract the RHS from the LHS:
\begin{align*}
    \frac{2}{3}\log\left(\frac{2}{3x}\right) + \frac{1}{3}\log\left(\frac{1}{3(1-x)}\right) - \frac{1}{3}\log\left(\frac{1}{2x}\right)
    &= \frac{1}{3}\log\left(\frac{2}{9x(1-x)}\right) + \frac{1}{3}\log\left(\frac{4}{3}\right)\\
    &\geq \frac{1}{3}\log\left(\frac{8}{9}\right) + \frac{1}{3}\log\left(\frac{4}{3}\right)\\
    &= \frac{1}{3}\log\left(\frac{32}{27}\right) > 0
\end{align*}
where we used the fact that $x(1-x) \leq 0.25$.\qed
}

The following lemma will be useful in establishing Theorem~\ref{thm:lb_restricted_class},
and the high probability result in Theorem~\ref{thm:partitioning}.
It uses the above change of measure lemma to argue that the probability of pulling an arm a large
number of times will be large.

\begin{lemma}\label{lemma:suboptimal_elim_cost_prob}
Let $\mathcal{P}$ be the set of $n$-armed bandits with
normally distributed rewards, whose mean is in $[0, 1]$ and variance is $\sigma^2$.
Let $A$ be a $(\epsilon, \delta)$-PAC algorithm on $\mathcal{P}$, with $\delta \leq \frac{1}{6}$.
Define $c = \frac{2}{3}\log\left(\frac{1}{2\delta}\right)$ and $\hat{\Delta}_i(\nu) =
\frac{\Delta_i(\nu) + \epsilon}{\sigma}$. Let $\nu \in \mathcal{P}$, such that $\mu_{[1]} \leq 1 -
\epsilon$ and $\Delta_1 \geq \epsilon$.
For any $i \in \{1,\dots,n\}\setminus\{[1]\}$,
\begin{align*}
\mathbb{P}_{A, \nu}\left(\widetilde{N}_{i,T}\geq c\hat{\Delta}_i(\nu)^{-2}\right) &\geq \frac{1}{6}.
\end{align*}
\end{lemma}
\proof{
Define $\alpha =\frac{2}{3}$
Assume, by way of contradiction that
\begin{align*}
    \mathbb{P}_{A, \nu}\left(\widetilde{N}_{i,T}  <  c\hat{\Delta}_i(\nu)^{-2}\right) \geq \frac{5}{6}.
\end{align*}
Because $A$ is $(\epsilon, \delta)$-PAC, we also have that 
\begin{align*}
    \mathbb{P}_{A, \nu}\left(\hat{I} = [1]\right) \geq 1 - \delta.
\end{align*}
}
Since $\delta \leq \frac{1}{6}$.
the above two conclusions imply that:
\begin{align*}
    \mathbb{P}_{A, \nu}\left(\hat{I} = [1]\cap \widetilde{N}_{i,T} <  c\hat{\Delta}_i(\nu)^{-2}\right) \geq \alpha
\label{eqn:altalgclaimone} \numberthis
\end{align*}
Let $\nu'$ be an alternate model with
\begin{align*}
 \nu_j' &=
  \begin{cases}
   \nu_j    & \text{if } j\neq i\\
   \mathcal{N}(\mu_{[1]} + \epsilon, \sigma^2)       & \text{if } j = i
  \end{cases}
    .
\end{align*}
Again, since $A$ is $(\epsilon, \delta)$-PAC,
\begin{align*}
    \mathbb{P}_{A, \nu'}\left(\hat{I} = [1]\right) < \delta \implies \mathbb{P}_{A, \nu'}\left(\hat{I} = [1]\cap \widetilde{N}_{i,T} <  c\hat{\Delta}_i(\nu)^{-2}\right) < \delta
\label{eqn:altalgclaimtwo} \numberthis
\end{align*}
No, we will consider the following alternative
 algorithm $\widetilde{A}$ that operates exactly as $A$, but ensures that
$\widetilde{N}_i \leq c\hat{\Delta}_i(\nu)^{-2}$.
If a decision cannot be arrived with this many pulls,
it stops and outputs "FAIL". Because $c\hat{\Delta}_i(\nu)^{-2}$ is used as the cutoff for any
problem (i.e., the cutoff is still $c\hat{\Delta}_i(\nu)^{-2}$ for a different problem $\nu'$), it
does not require knowledge of the arm gaps. Then, by~\eqref{eqn:altalgclaimone}
and~\eqref{eqn:altalgclaimtwo} we have,
\begin{align*}
    \mathbb{P}_{\widetilde{A}, \nu}(\hat{I} = [1]) \geq \alpha,
    \hspace{0.4in}
    \mathbb{P}_{\widetilde{A}, \nu'}(\hat{I} = [1]) < \delta
\end{align*}

Applying Lemma~\ref{lemma:lr_lower_bound}, we have
\begin{align*}
    \mathbb{E}_{\widetilde{A}, \nu}\left[{\widetilde{N}_{i,T}}\right]D_{KL}(\mathcal{N}(\mu_i, \sigma^2), \mathcal{N}(\mu_{[1]} + \epsilon, \sigma^2) &\geq d\left(\mathbb{P}_{\widetilde{A}, \nu}\left(\hat{I} = [1]\right), \mathbb{P}_{\widetilde{A}, \nu'}\left(\hat{I} = [1]\right)\right)\\
    &\geq d\left(\frac{2}{3}, \delta\right)\\
    &\geq \frac{1}{3}\log\left(\frac{1}{2\delta}\right) = \frac{c}{2}
\end{align*}

Moving the KL-divergence term to the RHS, we have
\begin{align*}
    \mathbb{E}_{\widetilde{A}, \nu}\left[{\widetilde{N}_{i,T}}\right] &\geq \frac{c}{2D_{KL}(\mathcal{N}(\mu_i, \sigma^2), \mathcal{N}(\mu_{[1]} + \epsilon, \sigma^2)}\\
    &\geq c\hat{\Delta}_i(\nu)^{-2}
\end{align*}
resulting in a contradiction as $\mathbb{E}_{\widetilde{A}, \nu}\left[{\widetilde{N}_{i,T}}\right] < c\hat{\Delta}_i(\nu)^{-2}$ by design.
\qed

\begin{remark}
The high probability lower bound also applies to arm $[1]$, if it is the only $\epsilon$-optimal arm in $\nu$. The proof is almost identical, but the alternate model pushes down $\mu_{[1]}$ by $\Delta_1 + \epsilon$ instead, making it no longer $\epsilon$-optimal. This requires the additional assumption that $\mu_{[2]} \geq \epsilon$.
\end{remark}





\subsection{Proof of Theorem~\ref{thm:partitioning}}
\proof{

\textbf{Expectation lower bound:}
Here, we will prove the lower bound in Theorem~\ref{thm:partitioning} on the expected number of pulls of an algorithm.
Let $\nu \in \mathcal{P}$, such that $\mu_{[1]} \leq 1 - \epsilon$ and $\mu_{[1]} - \epsilon \geq \mu_{[2]} \geq \epsilon$. Furthermore, let $\nu_i$ be a $\sigma^2$-variance normal distribution. Fix $i\in [n]$ such that $i \neq [1]$ (arm $i$ is not $\epsilon$-optimal). Let $\nu_{k + \alpha}$ denote the distribution $\nu_k$ shifted up by $\alpha$.
Define the alternative model $\nu'$, where $\nu_j' = \nu_j$ for all $j \neq i$ and $\nu_i' = \nu_{[1] + \epsilon}$.
Observe that in $\nu'$, the only $\epsilon$-optimal arm is arm $i$. Let $A$ be an $(\epsilon, \delta)$-PAC algorithm. Then,
\begin{align*}
    \mathbb{P}_{A,\nu}\left(\hat{I} = i\right) &\leq \delta\\
    \mathbb{P}_{A,\nu'}\left(\hat{I} = i\right) &\geq 1 - \delta
\end{align*}
Applying Lemma~\ref{lemma:exp_kl_bound}, we have that
\begin{align*}
    \mathbb{E}_{A,\nu}\left[\widetilde{N}_{i,T}\right] &\geq \frac{d(\delta, 1 - \delta)}{D_{KL}(\nu_i, \nu_i')}\\
    &= \frac{2\sigma^2 d(\delta, 1 - \delta)}{(\Delta_i(\nu) + \epsilon)^2}\\
    &\geq \frac{2\sigma^2}{(\Delta_i(\nu) + \epsilon)^2}\log\left(\frac{1}{2.4\delta}\right)
\end{align*}
where we use the definition of the KL-divergence of two normal distributions in the first step and the fact that $d(\delta, 1 - \delta) \geq \log\left(\frac{1}{2.4\delta}\right)$ in the last step.

We can apply this to any suboptimal $i$, and perform an identical argument for the optimal arm $[1]$ by shifting down $\nu_{[1]}$ by $\Delta_1 + \epsilon$ ($\nu_{[1]}' = \nu_{[1] - \Delta_1 - \epsilon}$). We sum over the individual arm's pulls to get:
\begin{align*}
    \mathbb{E}_{A,\nu}\left[\widetilde{N}_{T}\right] &= \sum_{i\in[n]}\mathbb{E}_{A,\nu}\left[\widetilde{N}_{i,T}\right]\\
    &\geq \sum_{i\in[n]}\frac{2\sigma^2}{(\Delta_i(\nu) + \epsilon)^2}\log\left(\frac{1}{2.4\delta}\right)\numberthis\label{eq:pull_exp_lb} \\
    &= 2\sigma^2\log\left(\frac{1}{2.4\delta}\right)\Hcompl(\nu)
\end{align*}
Now, let $\nu \in \mathcal{P}_\gamma$, for any non-empty $\mathcal{P}_{\gamma}$, such that $\Delta_i = \epsilon^{\frac{\gamma_i}{T}}$, $\mu_{[1]} \leq 1 - \epsilon$, and $\mu_{[1]} - \epsilon \geq \mu_{[2]} \geq \epsilon$. Since $\Hcompl(\nu)$ is a decreasing function of each of the arm gaps, $\nu = \arg\max_{\nu\in\mathcal{P}_\gamma}\Hcompl(\nu)$. Applying inequality~\ref{eq:pull_exp_lb}, we have that:
\begin{align*}
    \mathbb{E}_{A,\nu}\left[\widetilde{N}_{T}\right] &\geq \sum_{i\in[n]}\frac{2\sigma^2}{(\Delta_i(\nu) + \epsilon)^2}\log\left(\frac{1}{2.4\delta}\right)\\
    &= 2\sigma^2\log\left(\frac{1}{2.4\delta}\right)\Hcompl(\nu)\\
    &\geq \frac{\sigma^2}{2}\log\left(\frac{1}{2.4\delta}\right)\sum_{i\in[n]}\epsilon^{\frac{2\gamma_i}{T}}
\end{align*}
In the last step, we used the fact that $\epsilon \leq \epsilon^{\frac{\gamma_i}{T}}$.

\textbf{Probability lower bound:}
By Lemma~\ref{lemma:suboptimal_elim_cost_prob} and because we assume $n = 2$ in this part, we use the fact that $\Hcompl(\nu) = 2 \frac{1}{(\Delta_1 + \epsilon)^2}$ to show that
\begin{align*}
    \mathbb{P}_{A,\nu}\left(\widetilde{N}_{T} > \frac{2\sigma^2}{3}\log\left(\frac{1}{2\delta}\right)\frac{1}{(\Delta_1 + \epsilon)^2}\right) &\geq 
    \mathbb{P}_{A,\nu}\left(\widetilde{N}_{[2], T} > \frac{2\sigma^2}{3}\log\left(\frac{1}{2\delta}\right)\frac{1}{(\Delta_1 + \epsilon)^2}\right)\\
    &=
    \mathbb{P}_{A,\nu}\left(\widetilde{N}_{[2], T} > \frac{\sigma^2}{3}\log\left(\frac{1}{2\delta}\right)\Hcompl(\nu)\right)\\
    &\geq \frac{1}{6}
\end{align*}
As in the previous part, for nonempty partition $P_\gamma$, we can find $\nu$ with $\Delta_1 = \Delta_2 = \epsilon^{\frac{\gamma_1}{T}}$, which maximizes $\Hcompl(\nu)$ over $\mathcal{P}_\gamma$. Plugging this into the above inequality yields

\begin{align*}
    \mathbb{P}_{A,\nu}\left(\widetilde{N}_{T} > \frac{2\sigma^2}{3}\log\left(\frac{1}{2\delta}\right)\frac{1}{(\epsilon^{\frac{\gamma_1}{T}} + \epsilon)^2}\right)
    &\geq \frac{1}{6}
\end{align*}
Since $\epsilon^{\frac{\gamma_1}{T}} \geq \epsilon$, this also means that

\begin{align*}
    \mathbb{P}_{A,\nu}\left(\widetilde{N}_{T} > \frac{\sigma^2}{6}\log\left(\frac{1}{2\delta}\right)\epsilon^{-\frac{2\gamma_1}{T}}\right)
    &\geq \frac{1}{6}
\end{align*}
\qed
}






\subsection{Proof of Theorem~\ref{thm:lb_restricted_class}}
We will now prove the high probability lower bound over the restricted class of algorithms $\widetilde{\alg}$ in Theorem~\ref{thm:lb_restricted_class}. We will first require the following lemma, which is used to bound the best possible way to schedule batches of parallel arm pulls.
First, define the following set $\Qcal$.
Here,
 $N_{\mathcal{Q}}, R_{\mathcal{Q}} \geq 1$ are quantities that we will define shortly.
\[
\Qcal = \left\{Q = \left\{Q^{(1)}, Q^{(2)},\ldots, Q^{(T)}\right\} \in \mathbb{R}_{+}^T: Q^{(i)}
\leq Q^{(j)} \text{ for } i \leq j, Q^{(1)}\geq R_{\mathcal{Q}}, \sum_{s\in[T]}Q^{(s)} \geq
N_{\mathcal{Q}}\right\}.
\]
As we will see shortly, for appropriately chosen $
N_{\mathcal{Q}}, R_{\mathcal{Q}}$, an $(\epsilon,\delta)$-PAC
algorithm in $\widetilde{\mathcal{A}}$
will be a subset of the algorithms that choose 
chose some $Q\in\mathcal{Q}$ and then set
$\{Q_{i,t}\}_{t=1}^T$ to be some permutation of $Q$.
%
Next,
For $Q\in\mathcal{Q}$, let $2^Q$ be all the subsets of $Q$, (therefore $|2^Q| = 2^T$).
Now, define the following function $\phi: 2^Q \rightarrow \mathbb{R}_+$ s.t.
\begin{align*}
    \phi(\emptyset) &= R_{\mathcal{Q}},\\
    \phi(\bar{Q}) &= \sum_{q\in\bar{Q}}q. 
\end{align*}

The following technical result about $\Qcal$ will be useful going forward.
Its proof is given in Appendix~\ref{subsec:scheduling_proof}.

\begin{lemma}\label{lemma:scheduling2}
Assume $R_{\mathcal{Q}} \geq 1$ and $\left(\frac{N_{\mathcal{Q}}}{R_{\mathcal{Q}}}\right)^{1/T} > 4$.
We have,
\begin{align*}
    \inf_{Q\in\mathcal{Q}} \sup_{x\in[R_{\mathcal{Q}},N_{\mathcal{Q}}]} \min_{\bar{Q}\in 2^Q,\; \phi(\bar{Q}) > x} \frac{\phi(\bar{Q})}{x} &\geq \frac{1}{8}\left(\frac{N_{\mathcal{Q}}}{R_{\mathcal{Q}}}\right)^{1/T}
\end{align*}
\end{lemma}

We will now prove Theorem~\ref{thm:lb_restricted_class}.

\proof{
By Lemma~\ref{lemma:lr_lower_bound}, we have that w.p. at least $1/6$,
\begin{align*}
    \widetilde{N}_{i,T} &> \frac{\sigma^2}{3}\log\left(\frac{1}{2\delta}\right)\frac{1}{(\Delta_i + \epsilon)^2}
\end{align*}
Define $N_0 = \frac{1}{12}\log\left(\frac{1}{2\delta}\right)\frac{1}{\epsilon^2}$. Since the pull lower bound has to be true for all $\nu\in\mathcal{P}$, there exists $\nu \in \mathcal{P}$ such that $\widetilde{N}_{i,T} > N_0$ (by setting $\Delta_i = \epsilon$). 
We must show that such a $\nu$ exists. Consider $\nu = \left(\mathcal{N}(\mu_1, \sigma^2), \mathcal{N}(\mu_1 - \epsilon, \sigma^2)\right)$ where $\nu\in\mathcal{P}$. In this problem, $[1] = 1$, and arm $1$ is the only $\epsilon$-optimal arm.

Define $\widetilde{Q}_{i,t} = \sum_{s=1}^t Q_{i,s}$. Since $A$ cannot pull arm $i$ more than $\widetilde{Q}_{i,T}$ times,
it is necessarily the case that 
    $\widetilde{Q}_{i,T} > N_0$
in order to satisfy the lower bound's condition that the algorithm will pull at least $N_0$ times on
the above problem $\nu$ with probability  at least $\frac{1}{6}$.
Now define
\begin{align*}
    P_{i,\nu} = \frac{\sigma^2}{3}\log\left(\frac{1}{2\delta}\right)\frac{1}{(\Delta_i + \epsilon)^2}
\end{align*}
Recall from Lemma~\ref{lemma:suboptimal_elim_cost_prob}, that with probability at least $\frac{1}{6}$, $A$ must pull arm $i$ more than $P_{i,\nu}$ times. Define 
\begin{align*}
    P_{\rm min, i, \nu} = \min_{\bar{Q}\in 2^{\{Q_{i,t}\}_{t=1}^T},\; \phi(\bar{Q}) > P_{i,\nu}} \phi(\bar{Q}).
\end{align*}
This is the most efficient way to pull arm $i$ over $P_{i,\nu}$ times, given the predetermined sequence $\{Q_{i,t}\}_{t=1}^T$.
%
Now, observe that we can find a distribution $\nu\in \mathcal{P}$ that has $\Delta_i \in [\epsilon, 1]$ such that 
\begingroup
\allowdisplaybreaks
\begin{align*}
    P_{i,\nu} = \frac{\sigma^2}{3}\log\left(\frac{1}{2\delta}\right)\frac{1}{(\Delta_i + \epsilon)^2} &\in \frac{\sigma^2}{3}\log\left(\frac{1}{2\delta}\right)\left[\frac{1}{(1 + \epsilon)^2}, \frac{1}{4\epsilon^2}\right]\\
    &= \frac{\sigma^2}{3(1 + \epsilon)^2}\log\left(\frac{1}{2\delta}\right)\left[1, \frac{(1 + \epsilon)^2}{4\epsilon^2}\right]\\
\end{align*}
\endgroup
We will now apply Lemma~\ref{lemma:scheduling2} to this interval, letting $R_\mathcal{Q} = \frac{\sigma^2}{3(1 + \epsilon)^2}\log\left(\frac{1}{2\delta}\right)$ and $N_\mathcal{Q} = \frac{\sigma^2}{3(1 + \epsilon)^2}\log\left(\frac{1}{2\delta}\right)\frac{(1 + \epsilon)^2}{4\epsilon^2}$. The precondition for Lemma~\ref{lemma:scheduling2} is satisfied, because $\epsilon \leq 2^{-(T+1)}$ by assumption. The result states that no matter how the $Q_{i,t}$ values are selected, we will be able to find $\nu\in\mathcal{P}$ such that
\begingroup
\allowdisplaybreaks
\begin{align*}
    \frac{P_{\rm min, i, \nu}}{P_{i,\nu}} &\geq \frac{1}{8}\left(\frac{1 + \epsilon}{2\epsilon}\right)^{\frac{2}{T}}\\
    &\geq \frac{1}{8}\left(\frac{1}{2\epsilon}\right)^{\frac{2}{T}}\\
    &\geq \frac{1}{32}\epsilon^{-\frac{2}{T}}
\end{align*}
\endgroup
for any $i$. The last step uses the fact that $T \geq 1$.
Since $P_{i,\nu} = \frac{\sigma^2}{6}\log\left(\frac{1}{2\delta}\right)\Hcompl(\nu)$, there will always exist $\nu\in\mathcal{P}$ such that
\begin{align*}
    \frac{P_{\rm min, i, \nu}}{\Hcompl(\nu)} &\geq \frac{\sigma^2}{192}\log\left(\frac{1}{2\delta}\right)\epsilon^{-\frac{2}{T}}
\end{align*}
Since $\widetilde{N}_{i,T} > P_{i,\nu}$ with probability at least $1/6$, and since in these cases, the only way to pull arm $i$ at least $P_{i,\nu}$ times is to pull at least $P_{\rm min, i, \nu}$ times,
\begin{align*}
    \frac{\widetilde{N}_{T}}{\Hcompl(\nu)} \geq \frac{\widetilde{N}_{i,T}}{\Hcompl(\nu)} &\geq \frac{\sigma^2}{192}\log\left(\frac{1}{2\delta}\right)\epsilon^{-\frac{2}{T}}
\end{align*}
with probability at least $\frac{1}{6}$.
\qed


}

\input{scheduling_lemma}

\subsection{Proof of Theorem~\ref{thm:lb_n2_T2}}

Finally, we will prove Theorem~\ref{thm:lb_n2_T2}.
We will require the following lemma from~\citet{tsybakov2008introduction},
which can be interpreted as a high probability version of Pinsker's inequality.

\begin{lemma}~\label{lemma:tsybakov}
(\citet{tsybakov2008introduction}, Lemmas 2.1 and 2.6) Let $\mathbb{P},\mathbb{Q}$ be probabilities
such that $\mathbb{Q}$ is absolutely continuous with respect to $\mathbb{P}$ and with support
$\Xcal$. Let $\phi:\Xcal\rightarrow \{0, 1\}$. 
Then,
\begin{align*}
    \mathbb{P}(\phi(X)=1) + \mathbb{Q}(\phi(X)=1) &\geq
\frac{1}{2}\exp\left(-D_{KL}(\mathbb{P},\mathbb{Q})\right).
\end{align*}
\end{lemma}

\paragraph{Proof of Theorem~\ref{thm:lb_n2_T2}:}

We will consider 4 problems with 2 arms, $\nu^{(i)} = (\nu^{(i)}_1, \nu^{(i)}_2)$,
where $i\in[3]$. $\nu^{(i)}$ is a 2-armed bandit problem with $\nu_j^{(i)} = \mathcal{N}\left(\mu_j^{(i)}, \sigma^2\right)$ and $\mu_2^{(i)} = \frac{1}{2}$ for all $i$. Let
\begin{align*}
 \mu_1^{(1)} &= \frac{1}{2} - \epsilon\\
 \mu_1^{(2)} &= \frac{1}{2} - \epsilon - \sqrt{\epsilon}\\
 \mu_1^{(3)} &= 0
\end{align*}
Since $\epsilon < \frac{1}{10}$, we have that $\mu_2 \geq \mu_1^{(1)} \geq \mu_1^{(2)} \geq
\mu_1^{(3)}$. We will assume that $\mu_2^{(i)}$ is known, so the algorithm only needs to pull arm
$1$ to decide whether $\mu_1 \leq \frac{1}{2} - \epsilon$, $\mu_1 \in (\frac{1}{2} - \epsilon,
\frac{1}{2} + \epsilon)$, or $\mu_1 \geq \frac{1}{2} + \epsilon$ to make its decision.
We will assume, for simplicity, that the algorithm chooses the number of pulls in the first round
deterministically (see Remark~\ref{rem:lb_deterministic} for more details).
Therefore, 
it plays a certain number of times on the first round without prior information, and then chooses how
many times to pull in the second round based on information obtained in the first round.

\paragraph{Part 1:}
We will first prove the following claim. For any bandit model $\nu$, with $\mu_1 \leq \mu_2 - \epsilon$, we have that
\begin{align*}
    \mathbb{P}_{A,\nu}\left(\widetilde{N}_{12}\geq \frac{2\sigma^2}{25\Delta(\nu)^2}\right) \geq \frac{7}{8}
\end{align*}
where $\Delta(\nu) = \mu_2(\nu) - \mu_1(\nu)$. This is equivalent to showing that:
\begin{align*}
    \mathbb{P}_{A,\nu}\left(\widetilde{N}_{12}< \frac{2\sigma^2}{25\Delta(\nu)^2}\right) < \frac{1}{8}
\end{align*}
This part of the proof will be similar to the intuition used in the proof of
Lemma~\ref{lemma:suboptimal_elim_cost_prob}.
Recall that $\delta < \frac{1}{32}$ by our assumptions. Assume by way of contradiction that
\begin{align*}
    \mathbb{P}_{A,\nu}\left(\widetilde{N}_{12}< \frac{2\sigma^2}{25\Delta(\nu)^2}\right) \geq \frac{3}{32} + \delta
\end{align*}
Since $A$ is $(\epsilon, \delta)$-PAC,
\begin{align*}
    \mathbb{P}_{A,\nu}\left(\hat{I} = 2\right) &\geq 1 - \delta
\end{align*}
which implies that
\begin{align*}
    \mathbb{P}_{A,\nu}\left(\hat{I} = 2\cap \widetilde{N}_{12}< \frac{2\sigma^2}{25\Delta(\nu)^2}\right) \geq \frac{3}{32}
\end{align*}
Consider an alternative model $\nu'$ where $\mu_1 = \frac{1}{2} + \epsilon$. Then
\begin{align*}
    \mathbb{P}_{A,\nu'}\left(\hat{I} = 2\right) &< \delta \implies \mathbb{P}_{A,\nu'}\left(\hat{I} = 2\cap \widetilde{N}_{12}< \frac{\sigma^2}{25\Delta(\nu)^2}\right) < \delta
\end{align*}
Let $\widetilde{A}$ be an alternate algorithm that runs exactly as $A$, but terminates just before reaching $\frac{2\sigma^2}{25\Delta(\nu)^2}$ pulls if necessary.
Observe that
\begin{align*}
    \mathbb{P}_{\widetilde{A},\nu}\left(\hat{I} = 2\right) &\geq \frac{3}{32}\\
    \mathbb{P}_{\widetilde{A},\nu'}\left(\hat{I} = 2\right) &\leq \delta
\end{align*}

By Lemma~\ref{lemma:exp_kl_bound}, and using the fact that $\delta < 1/32$, we have
\begin{align*}
    \mathbb{E}_{\widetilde{A}, \nu}\left[\widetilde{N}_{12}\right]D_{KL}(\nu_1, \nu_1') \geq d\left(\frac{3}{32}, \delta\right)
    \geq d\left(\frac{3}{32}, \frac{1}{32}\right)
    \geq \frac{1}{25}.
\end{align*}
Since $D_{KL}(\nu_1, \nu_1') = \frac{\Delta(\nu)^2}{2\sigma^2}$, this implies that
\begin{align*}
     \mathbb{E}_{\widetilde{A}, \nu}\left[\widetilde{N}_{12}\right] &\geq \frac{2\sigma^2}{25\Delta(\nu)^2}
\end{align*}
resulting in a contradiction.

Recall that an algorithm chooses the number of pulls for the second round at the end of the first
round.
Let $\mathcal{E}$ denote the $\mathcal{F}_1$-measurable event that $\widetilde{N}_{12} >
\frac{2\sigma^2}{25(2\epsilon)^2} = \frac{\sigma^2}{50\epsilon^2}$, which is equivalent to the
event that $N_{12} > \frac{\sigma^2}{50\epsilon^2} - N_{11}$.


\paragraph{Part 2:}
We will prove the following for problem $\nu^{(2)}$. If $N_{11} \leq \frac{2\sigma^2\log(2)}{(\mu_{1}^{(1)} - \mu_{1}^{(2)})^2} = \frac{2\sigma^2\log(2)}{\epsilon}$, then $\mathbb{P}_{A, \nu^{(2)}}(E) \geq \frac{1}{4}$.

Assume that $N_{11} \leq \frac{2\sigma^2\log(2)}{(\mu_{1}^{(1)} - \mu_{1}^{(2)})^2} =
\frac{2\sigma^2\log(2)}{\epsilon}$.
Recall that $\mathcal{E}$ is $\mathcal{F}_1$-measurable. Then applying
Lemma~\ref{lemma:tsybakov} with $\phi(\cdot) = \indfone_\Ecal(\cdot)$, we have,
\begin{align*}
    \mathbb{P}_{A,\nu^{(1)}}\left(E^c\right) + \mathbb{P}_{A,\nu^{(2)}}\left(E\right) &\geq \frac{1}{2}\exp\left(D_{KL}(\mathbb{P}_{A, \nu^{(1)}}, \mathbb{P}_{A, \nu^{(2)}})\right)\\
    &\geq \frac{1}{2}\exp\left(-\log(2)\right)\\
    &\geq \frac{1}{4}
\end{align*}
where we use the lower bound on $N_{11}$ in the second inequality. By part 1, $\mathbb{P}_{A,\nu^{(1)}}\left(E^c\right) \leq \frac{1}{8}$. So,
\begin{align*}
    \mathbb{P}_{A,\nu^{(2)}}\left(E\right) &\geq \frac{1}{8}
\end{align*}

\paragraph{Part 3:}
Suppose the algorithm had chosen $N_{11} \geq \frac{\sigma^22\log(2)}{\epsilon}$. Then, on problem $\nu^{(3)}$, $\Hcompl(\nu^{(3)} = 2\frac{1}{(1/2 + \epsilon)^2} \leq 2\frac{1}{\frac{1}{4}} = 8$. So,
\begin{align*}
    \frac{\widetilde{N}_{12}}{\Hcompl(\nu)} &\geq \frac{\sigma^22\log(2)}{8\epsilon}\\
    &= \frac{\sigma^2}{\epsilon}\frac{\log(2)}{4}.
\end{align*}

Suppose the policy had chosen $N_{11} < \frac{\sigma^22\log(2)}{\epsilon}$. Recall that on problem $\nu^{(2)}$, $\Hcompl(\nu^{(2)}) = 2 \frac{1}{(2\epsilon + \sqrt{\epsilon})^2} \leq \frac{2}{\epsilon}$. From part 2, with probability at least $1/8$,
\begin{align*}
    \frac{\widetilde{N}_{12}}{\Hcompl(\nu^{(2)})} &\geq \frac{\sigma^2}{50\epsilon^2} \times \frac{\epsilon}{2}
\end{align*}
\qed

\begin{remark}
\label{rem:lb_deterministic}
Our proof above assumes that the number of pulls in the first round is chosen deterministically.
If it were randomized based on some (external) source of randomness $U$,
then the statement of the theorem holds for every possible realization of $U$.
The second statement in Corollary~\ref{corollary:expectation_lb} holds, simply by taking an
expectation over $U$.
\end{remark}

\subsection{Proof of the change of measure lemma}
\label{sec:comlemma}

In this subsection, we prove Lemma~\ref{lem:comlemma}.
The proof uses essentially the same intuition as the proof of Lemma 1
in~\citet{kaufmann2016complexity}, but due to the changes in the set up, we need to verify some
intermediate results.
The following lemma is an adaptation of a claim in Lemma 18 of~\citet{kaufmann2016complexity}.

\begin{lemma}\label{lemma:lr_g}
Fix an algorithm $A$ and consider any $t\in [T]$. Consider any function $g: \mathcal{S}_t \rightarrow \mathbb{R}$ that is measurable w.r.t. $\mathcal{F}_t$, where $\mathcal{S}_t$ is the space of all observations in $t$ rounds. Then,
\begin{align*}
    \mathbb{E}_{\nu'}\left[g(O_{1:t})\right] &= \mathbb{E}_{\nu}\left[g(O_{1:t})\exp\left(-L(A_{1:t}, O_{1:t})\right)\right]
\end{align*}
\end{lemma}

\proof{
Write $\mathcal{S}_t = \mathcal{G}^t$, where $\mathcal{G} = \bigcup_{i \in [n]}\bigcup_{l=0}^\infty \mathbb{R}^l$ is the space of possible observations for a single round. Let $\mathcal{A}_1 = \mathbb{N}^n$ be the space of possible actions. We will prove the claim via induction. Define $Y_{a,t} = \bigcup_{i=1}^n \{X_{i,t,j}\}_{j=1}^{N^a_{i,t}}$ to be the set of samples revealed at round $t$ if action $a$ was executed, where $N^a_{i,t}$ is the number of pulls for arm $i$ by action $a$.

Let $O_1 \in \mathcal{S}_1$.
\begingroup
\allowdisplaybreaks
\begin{align*}
    \mathbb{E}_{\nu'}\left[g(O_1)\right] &= \mathbb{E}_{\nu'}\left[\sum_{a\in\mathcal{A}_1} \mathbbm{1}(A_1=a)g(Y_{a,1})\right]\\
    &= \sum_{a\in\mathcal{A}_1} \mathbb{E}_{\nu'}\left[ \mathbbm{1}(A_1=a)g(Y_{a,1})\right]\\
    &= \sum_{a\in\mathcal{A}_1} \mathbb{E}_{\nu'}\left[ \mathbbm{1}(A_1=a)\mathbb{E}_{\nu'}\left[g(Y_{a,1})|\mathcal{F}_0\right]\right]\\
    &= \sum_{a\in\mathcal{A}_1} \mathbb{P}_{\nu'}\left(A_1=a\right)\mathbb{E}_{\nu'}\left[g(Y_{a,1})\right]\\
    &= \sum_{a\in\mathcal{A}_1} \mathbb{P}_{\nu}\left(A_1=a\right)\mathbb{E}_{\nu}\left[g(Y_{a,1})\prod_{i=1}^n\prod_{j=1}^{N_{i,1}^a}\frac{f_i'(X_{i,1,j})}{f_i(X_{i,1,j})}\right]\\
    &= \mathbb{E}_{\nu}\left[\sum_{a\in\mathcal{A}_1}\mathbbm{1}(A_1=a) g(Y_{a,1})\prod_{i=1}^n\prod_{j=1}^{N_{i,1}^a}\frac{f_i'(X_{i,1,j})}{f_i(X_{i,1,j})}\right]\\
    &= \mathbb{E}_{\nu}\left[g(O_1)\sum_{a\in\mathcal{A}_1}\mathbbm{1}(A_1=a)\prod_{i=1}^n\prod_{j=1}^{N_{i,1}^a}\frac{f_i'(X_{i,1,j})}{f_i(X_{i,1,j})}\right]\\
    &= \mathbb{E}_{\nu}\left[g(O_1) \sum_{a\in\mathcal{A}_1}\mathbbm{1}(A_1=a)\exp\left(\log\left(\prod_{i=1}^n\prod_{j=1}^{N_{i,1}^a}\frac{f_i'(X_{i,1,j})}{f_i(X_{i,1,j})}\right)\right)\right]\\
    &= \mathbb{E}_{\nu}\left[g(O_1) \exp\left(\sum_{a\in\mathcal{A}_1}\mathbbm{1}(A_1=a)\log\left(\prod_{i=1}^n\prod_{j=1}^{N_{i,1}^a}\frac{f_i'(X_{i,1,j})}{f_i(X_{i,1,j})}\right)\right)\right]\\
    &= \mathbb{E}_{\nu}\left[g(O_1) \exp\left(-L_1\right)\right]\\
\end{align*}
\endgroup

Now, assume that the statement holds for some $t\in[T-1]$. Let us prove that it holds for $t + 1$.
\begingroup
\allowdisplaybreaks
\begin{align*}
    \mathbb{E}_{\nu'}\left[g(O_{1:t})\right] &= \mathbb{E}_{\nu'}\left[\mathbb{E}_{\nu'}\left[g(O_{1:t+1})|\mathcal{F}_t\right]\right]\\
    &= \mathbb{E}_{\nu}\left[\mathbb{E}_{\nu'}\left[g(O_{1:t+1})|\mathbb{F}_t\right]\exp\left(-L_t\right)\right]\\
    &= \mathbb{E}_{\nu}\left[\sum_{a\in\mathcal{A}_1} \mathbbm{1}(A_{t+1}=a)\mathbb{E}_{\nu'}\left[g(O_{1:t}\cup Y_{a,t+1})|\mathbb{F}_t\right]\exp\left(-L_t\right)\right]\\
    &= \mathbb{E}_{\nu}\left[\sum_{a\in\mathcal{A}_1} \mathbbm{1}(A_{t+1}=a)\int_{\mathcal{G}} g(O_{1:t}\cup Y_{a,t+1})\prod_{i=1}^n\prod_{j=1}^{N_{i,1}^a}\frac{f_i'(x_{i,1,j})}{f_i(x_{i,1,j})}f_i(x_{i,1,j})d\lambda(o_{t+1})\exp\left(-L_t\right)\right]\\
    &= \mathbb{E}_{\nu}\left[\mathbb{E}_{\nu}\left[ g(O_{1:t+1})\exp\left(-L_{t+1}\right)\right]\right]\\
    &= \mathbb{E}_{\nu}\left[g(O_{1:t+1})\exp\left(-L_{t+1}\right)\right]
\end{align*}
\endgroup
\qed

}

\begin{lemma}\label{lemma:lr_event_prob}
Let $A$ be any algorithm and let $t \in [T]$. For all events $\mathcal{E} \in \mathcal{F}_t$,
\begin{align*}
    \mathbb{P}_{\nu'}(\mathcal{E}) &= \mathbb{E}_{\nu}\left[\mathbbm{1}_{\mathcal{E}}\exp(-L_t)\right].
\end{align*}
\end{lemma}
\proof{
This result follows by setting $g = \mathbbm{1}_{\mathcal{E}}$ in Lemma~\ref{lemma:lr_g}.\qed
}

\begin{lemma}
Let $A$ be any algorithm. Then,
\begin{align*}
    \mathbb{E}_{\nu}\left[L_t\right] &= \sum_{i\in[n]}\mathbb{E}_\nu\left[\widetilde{N}_{i,t}\right] D_{KL}(\nu_i, \nu_i').
\end{align*}
\end{lemma}

\proof{
\begingroup
\allowdisplaybreaks
\begin{align*}
    \mathbb{E}_\nu\left[L_t\right] &= \mathbb{E}_\nu\left[L(A_{1:t}, O_{1:t})\right]\\
    &= \sum_{i\in[n]}\sum_{s=1}^t\mathbb{E}_\nu\left[\sum_{l=1}^{N_{i,s}}\log\left(\frac{f_i(X_{i,s,l})}{f_i'(X_{i,s,l})}\right)\right]\\
    &= \sum_{i\in[n]}\sum_{s=1}^t\mathbb{E}\left[\mathbb{E}_\nu\left[\sum_{l=1}^{N_{i,t}}\log\left(\frac{f_i(X_{i,s,l})}{f_i'(X_{i,s,l})}\right)\right]|\mathcal{F}_{s-1}\right]\\
    &= \sum_{i\in[n]}\sum_{s=1}^t \mathbb{E}\left[N_{i,s}\right]D_{KL}(\nu_i, \nu_i')\\
    &= \sum_{i\in[n]}D_{KL}(\nu_i, \nu_i')\mathbb{E}\left[\widetilde{N}_{i,t}\right]
\end{align*}
\endgroup
\qed
}

This lemma characterizes the variation between expected number of pulls of any algorithm $A$ between any two bandit models $\nu$ and $\nu'$.

\begin{lemma}\label{lemma:lr_lower_bound}
Let $A$ be any algorithm and let $\nu, \nu'$ be two bandit models from $\mathcal{P}$ with $n$ arms such that for all $i\in[n]$, $\nu_i$ and $\nu'_i$ are absolutely continuous w.r.t. each other. Then, for every event $\mathcal{E}$ in $\mathcal{F}_t$,
\begin{align*}
    \mathbb{E}_{A,\nu}\left[L_t\right] &\geq d\left(\mathbb{P}_\nu(\mathcal{E}), \mathbb{P}_{\nu'}(\mathcal{E})\right).
\end{align*}
\end{lemma}

\proof{
This proof is identical the proof of Lemma 19 of \citet{kaufmann2016complexity}, except it uses $\mathcal{F}_t$ and Lemma~\ref{lemma:lr_event_prob}.\qed



}

We are now ready to prove Lemma~\ref{lem:comlemma}.

\paragraph{Proof of Lemma~\ref{lem:comlemma}:}
We will first show that $\mathbb{E}_{A,\nu}\left[L_t\right] = \sum_{i\in[n]}\mathbb{E}_\nu\left[\widetilde{N}_{i,t}\right]D_{KL}(\nu_i, \nu_i')$

\begingroup
\allowdisplaybreaks
\begin{align*}
    \mathbb{E}_{A,\nu}\left[L_t\right] &= \sum_{i=1}^n\sum_{s=1}^t\mathbb{E}_\nu\left[\sum_{k=1}^{N_{i,s}}\log\left(\frac{f_i(X_{i, s, k})}{f_i'(X_{i,s,k})}\right)\right]\\
    &= \sum_{i=1}^n\sum_{s=1}^t\mathbb{E}_\nu\left[\mathbb{E}_\nu\left[\sum_{k=1}^{N_{i,s}}\log\left(\frac{f_i(X_{i, s, k})}{f_i'(X_{i,s,k})}\right)|\mathcal{F}_{s - 1}\right]\right]\\
    &= \sum_{i=1}^n\sum_{s=1}^t\mathbb{E}_\nu\left[N_{i,s}\mathbb{E}_\nu\left[\log\left(\frac{f_i(X_{i, s, k})}{f_i'(X_{i,s,k})}\right)|\mathcal{F}_{s - 1}\right]\right]\\
    &= \sum_{i=1}^n\sum_{s=1}^t\mathbb{E}_\nu\left[N_{i,s}D_{KL}(\nu_i, \nu_i')\right]\\
    &= \sum_{i=1}^n\mathbb{E}_\nu\left[\sum_{s=1}^t N_{i,s}\right]D_{KL}(\nu_i, \nu_i')\\
    &= \sum_{i=1}^n \mathbb{E}_\nu\left[\widetilde{N}_{i,s}\right] D_{KL}(\nu_i, \nu_i')\\
\end{align*}
\endgroup
Applying Lemma~\ref{lemma:lr_lower_bound} finishes the proof.
\qed

\section{Baseline implementation details}

We use the following hyperparameters when implementing the baseline algorithms used for comparison.

\subsection{Top-$k$ $\delta$-Elimination with Limited Rounds:}
In order to implement this algorithm from~\citet{jin2019efficient}, we use the following hyperparameters, which provide $(\epsilon, \delta)$-PAC guarantees. We set $Q = \frac{57}{\epsilon^2}$, $k = 1$, $R = T$, and $S = [n]$.

\subsection{Aggressive Elimination}
This algorithm from~\citet{agarwal2017learning} assumes that $\Delta_1$ is known. Because it is not known to other algorithms, we instead set $\Delta_1 = \epsilon$ when running this algorithm. We set the initial set of candidates $S = [n]$, number of arms to output $k = 1$, and initial time $r = T$.

\subsection{$T = \infty$ (Sequential)}
Because the confidence intervals scale with $T$, they can become too aggressively large and conservative with $T$ is large. To implement this baseline, we instead use the confidence intervals from Batch Racing~\cite{jun2016top}, as these still provide $(\epsilon, \delta)$-PAC guarantees without suffering this problem. Arms are pulled one at a time, and the same elimination condition is used as in \algabbr{}.

\subsection{$T = 1$ (Passive)}
This algorithm simply runs \algabbr{} with $T=1$.

%% file: scheduling_lemma.tex
\subsubsection{Proof of Lemma~\ref{lemma:scheduling2}}~\label{subsec:scheduling_proof}

Let $Q\in\mathcal{Q}$ be given. We need to show that $\exists x\in [R_{\mathcal{Q}}, N_{\mathcal{Q}}]$ s.t.
\begin{align*}
    \min_{\bar{Q}\in 2^Q,\; \phi(\bar{Q}) \geq x} \frac{\phi(\bar{Q})}{x} &\geq \frac{1}{8}\left(\frac{N_{\mathcal{Q}}}{R_{\mathcal{Q}}}\right)^{1/T}\numberthis\label{eq:scheduling_1}
\end{align*}
We will show this via contradiction. Assume that Inequality~\ref{eq:scheduling_1} is not true. That is,
\begin{align*}
    \min_{\bar{Q}\in 2^Q,\; \phi(\bar{Q}) \geq x} \frac{\phi(\bar{Q})}{x} &< \frac{1}{8}\left(\frac{N_{\mathcal{Q}}}{R_{\mathcal{Q}}}\right)^{1/T}\numberthis\label{eq:scheduling_2}
\end{align*}
Denote $\beta = \frac{1}{4}\left(\frac{N_{\mathcal{Q}}}{R_{\mathcal{Q}}}\right)^{1/T}$, which is greater than $1$ by our assumption.

\paragraph{Part 1:}
We will first show that $Q^{(t)} < \beta\sum_{s = 1}^{t - 1} Q^{(s)}$. Assume instead that $Q^{(t)} \geq \beta\sum_{s = 1}^{t - 1} Q^{(s)}$. Then, by choosing $x = \sum_{s=1}^{t - 1}Q^{(s)} + 1$, we have that
\begingroup
\allowdisplaybreaks
\begin{align*}
    \min_{\bar{Q}\in 2^Q,\; \phi(\bar{Q}) \geq x} \frac{\phi(\bar{Q})}{x} &\geq \frac{Q^{(t)}}{\sum_{s=1}^{t - 1}Q^{(s)} + 1}\\
    &\geq \frac{\beta (\sum_{s = 1}^{t-1}Q^{(s)} + 1)}{\sum_{s=1}^{t - 1}Q^{(s)} + 1} - \frac{\beta}{\sum_{s=1}^{t - 1}Q^{(s)} + 1}\\
    &\geq \beta - \frac{\beta}{\sum_{s=1}^{t - 1}Q^{(s)} + 1}\\
    &\geq \frac{1}{2}\beta\\
    &\geq \frac{1}{8}\left(\frac{N_{\mathcal{Q}}}{R_{\mathcal{Q}}}\right)^{1/T}
\end{align*}
\endgroup
The first step uses the fact that the value for $\phi(\bar{Q})$, must exceed $ \sum_{s = 1}^{t-1}Q^{(s)} + 1$, and this cannot be done by picking a subset with sum less than $Q^{(t)}$. The second step plugs in the assumption that $Q^{(t)} \geq \beta\sum_{s = 1}^{t - 1} Q^{(s)}$. This contradicts our first assumption in Inequality~\ref{eq:scheduling_2}, so the assumption made in this part must be incorrect. Therefore, $Q^{(t)} < \beta\sum_{s = 1}^{t - 1} Q^{(s)}$.

Observe that for $t = 1$, this part translates to $Q^{(1)} < R_{\mathcal{Q}}\beta$.

\paragraph{Part 2:}
We now claim that:
\begin{align*}
    Q^{(1)} &< R_{\mathcal{Q}}\beta \leq R_{\mathcal{Q}}(\beta + 1)\\
    Q^{(2)} &< R_{\mathcal{Q}}\beta^2 \leq R_{\mathcal{Q}}(\beta + 1)^2\\
    Q^{(s)} &< R_{\mathcal{Q}}\beta^2(\beta + 1)^{s - 2} \leq R_{\mathcal{Q}}(\beta + 1)^s,\; \forall s\geq 3
\end{align*}
This can be shown by a simple inductive argument and the observation that
    $Q^{(s+1)} < Q^{(s-1)} + \beta Q^{(s)}$.
. Suppose the claim holds for some $s \geq 3$. Then
\begin{align*}
    Q^{(s+1)} &< Q^{(s-1)} + \beta Q^{(s)}\\
    &< R_{\mathcal{Q}}(\beta^3 + \beta^2)(\beta + 1)^{s - 2}\\
    &= R_{\mathcal{Q}}\beta^2(\beta + 1)^{s - 1}
\end{align*}
proving the claim.

\paragraph{Part 3:}
We have that
\begin{align*}
    N_{\mathcal{Q}} &\leq \sum_{t = 1}^T Q^{(t)}
    < R_{\mathcal{Q}}\sum_{t = 1}^T (\beta + 1)^t
    < 2R_{\mathcal{Q}}(\beta + 1)^T,
    < 2^{T + 1}R_{\mathcal{Q}}\beta^T\\
    &= 2^{T+1}R_{\mathcal{Q}}\left(\frac{1}{4}\left(\frac{N_{\mathcal{Q}}}{R_{\mathcal{Q}}}\right)^{1/T}\right)^T
    = \frac{1}{2^{T-1}}N_{\mathcal{Q}}
    < N_{\mathcal{Q}}.
\end{align*}
We use the fact that $\beta > 1$ in the third and fourth steps.
This results in a contradiction, as by definition, $\sum_{t = 1}^T Q^{(t)} \geq N_{\mathcal{Q}}$.
\qed